\documentclass{article}

\usepackage{mathtools}
\usepackage{amssymb}

\usepackage[preprint]{styles/neurips_2023}

\usepackage{ulem}
\usepackage[utf8]{inputenc} % allow utf-8 input
\usepackage[T1]{fontenc}    % use 8-bit T1 fonts
\usepackage{hyperref}       % hyperlinks
\usepackage{url}            % simple URL typesetting
\usepackage{booktabs}       % professional-quality tables
\usepackage{amsfonts}       % blackboard math symbols
\usepackage{nicefrac}       % compact symbols for 1/2, etc.
\usepackage{microtype}      % microtypography
\usepackage{xcolor}         % colors
\usepackage{algorithm,algpseudocode}

\usepackage{graphicx}
\usepackage{caption}
\usepackage{subcaption}

\usepackage[toc,page,header]{appendix}
\usepackage{minitoc}

\useunder{\uline}{\ul}{}

\usepackage{styles/jmei}

\title{SKI to go Faster: Accelerating Toeplitz Neural Networks via Asymmetric Kernels}

\author{%
  Alexander Moreno\thanks{All authors contributed equally. Alexander and Jonathan flipped for ordering. The vast majority of this work was done while the authors were at Luminous Computing.}\quad Jonathan Mei$^*$\quad Luke Walters$^*$\\
  \texttt{\{alexander.f.moreno,jonathanmei,lukecwalters\}@gmail.com}
}

\begin{document}
\doparttoc % Tell to minitoc to generate a toc for the parts
\faketableofcontents % Run a fake tableofcontents command for the partocs

\part{} % Start the document part
%\parttoc % Insert the document TOC

\maketitle

\begin{abstract}
Toeplitz Neural Networks (TNNs) \cite{qin2023toeplitz} are a recent impressive sequence model requiring $O(n\log n)$ computational complexity and $O(n)$ relative positional encoder (RPE) multi-layer perceptron (MLP) and decay bias calls. We aim to reduce both. We first note that the RPE is a non symmetric positive definite kernel and the Toeplitz matrices are pseudo-Gram matrices. Further 1) the learned kernels display spiky behavior near the main diagonals with otherwise smooth behavior; 2) the RPE MLP is slow. For bidirectional models, this motivates a sparse plus low-rank Toeplitz matrix decomposition. For the sparse component's action, we do a small 1D convolution. For the low rank component, we replace the RPE MLP with linear interpolation and use Structured Kernel Interpolation (SKI) \cite{wilson2015kernel} for $O(n)$ complexity. For causal models, ``fast'' causal masking \cite{katharopoulos2020transformers} negates SKI's benefits. Working in frequency domain, we avoid an explicit decay bias. To enforce causality, we represent the kernel via the real part of its frequency response using the RPE and compute the imaginary part via a Hilbert transform. This maintains $O(n \log n)$ complexity but achieves an absolute speedup. Modeling the frequency response directly is also competitive for bidirectional training, using one fewer FFT.
  We improve on speed and sometimes score on the Long Range Arena (LRA) 
 \cite{taylong}.
  %\ifx but learning both the real and imaginary components \fi 
  \ifx with long sequences in mind\fi
\end{abstract}

\begin{figure}[h!]
    \centering
    \begin{subfigure}[t]{0.45\linewidth}
        \centering
        \includegraphics[width=0.9\linewidth]{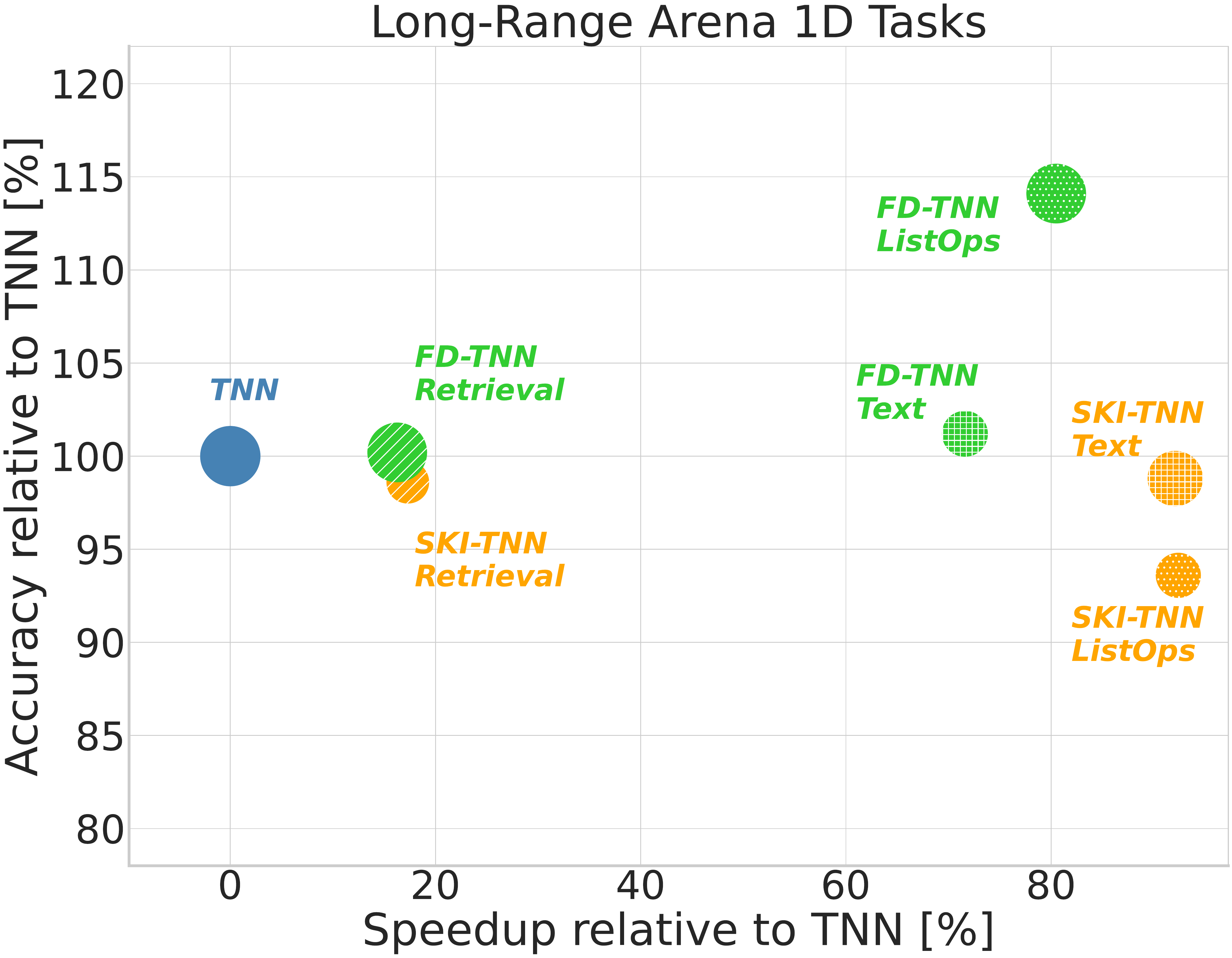}
        \caption{Long Range Arena (LRA).\label{fig:lra_bubble}}
    \end{subfigure}%
    \begin{subfigure}[t]{0.45\linewidth}
        \centering
    \includegraphics[width=0.9\linewidth]{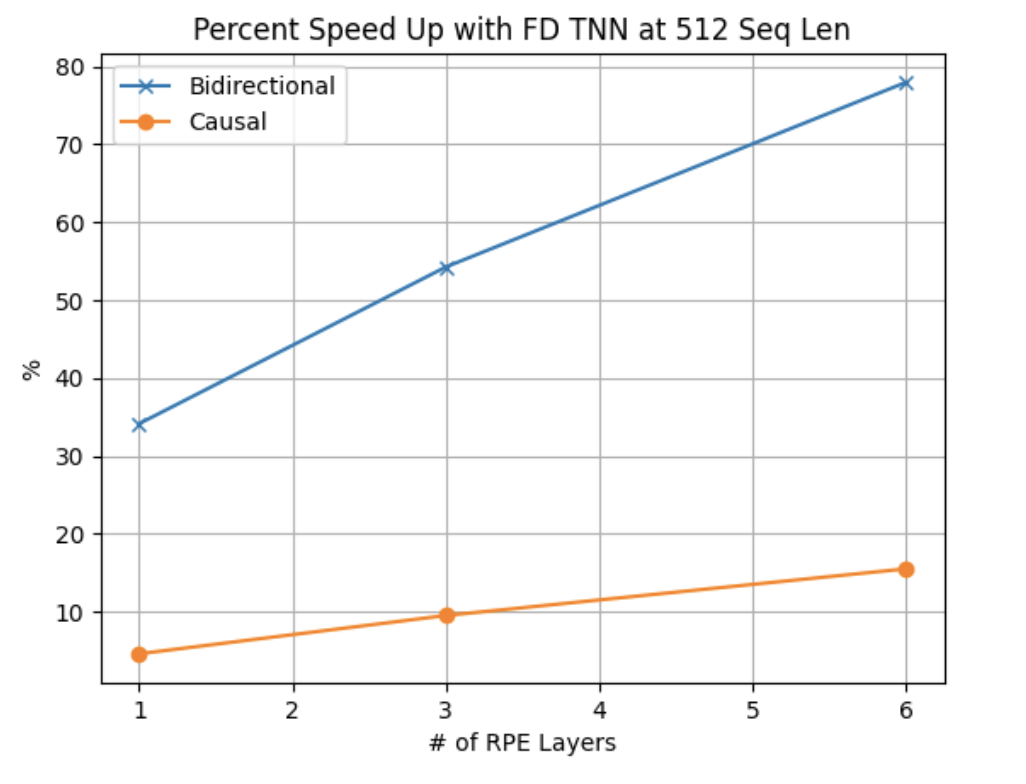}
    \caption{Pre-training speedups using Fourier Domain.\label{fig:extrapolation}}
    \end{subfigure}
    % \begin{subfigure}[t]{0.45\linewidth}
    %     \centering
    %     % \includegraphics[width=\linewidth]{figs/alm_extrap_all.png}
    %     \includegraphics[width=\linewidth]{figs/causal_vs_seq_len.pdf}
    %     \caption{Wikitext-103 Causal Pretraining.}

    % \end{subfigure}
    \caption{(a) In LRA, our approaches, SKI and FD-TNN are faster than TNNs for 1d tasks with strong LRA scores. Bubble sizes denote training model memory. (b) Our approach, FD-TNN, achieves substantial speed ups in iterations/sec for pre-training both causal and bidirectional models. Note that we do not include SKI-TNN in this plot as it does not use an MLP based RPE.}
    \label{fig:lra_bubble_and_extrapolation}
    %(b) In Wikitext-103 causal pretraining, our approach, FD TNN achieves equivalent perplexity vs inference length to TNN, with fast $\sim 10$-15\% training on a single A100 GPU at 512 sequence length.
\end{figure}

        \ifx
        \begin{figure}[h!]
            \centering
            \begin{subfigure}[t]{\linewidth}
                \centering
                \includegraphics[width=0.45\linewidth]{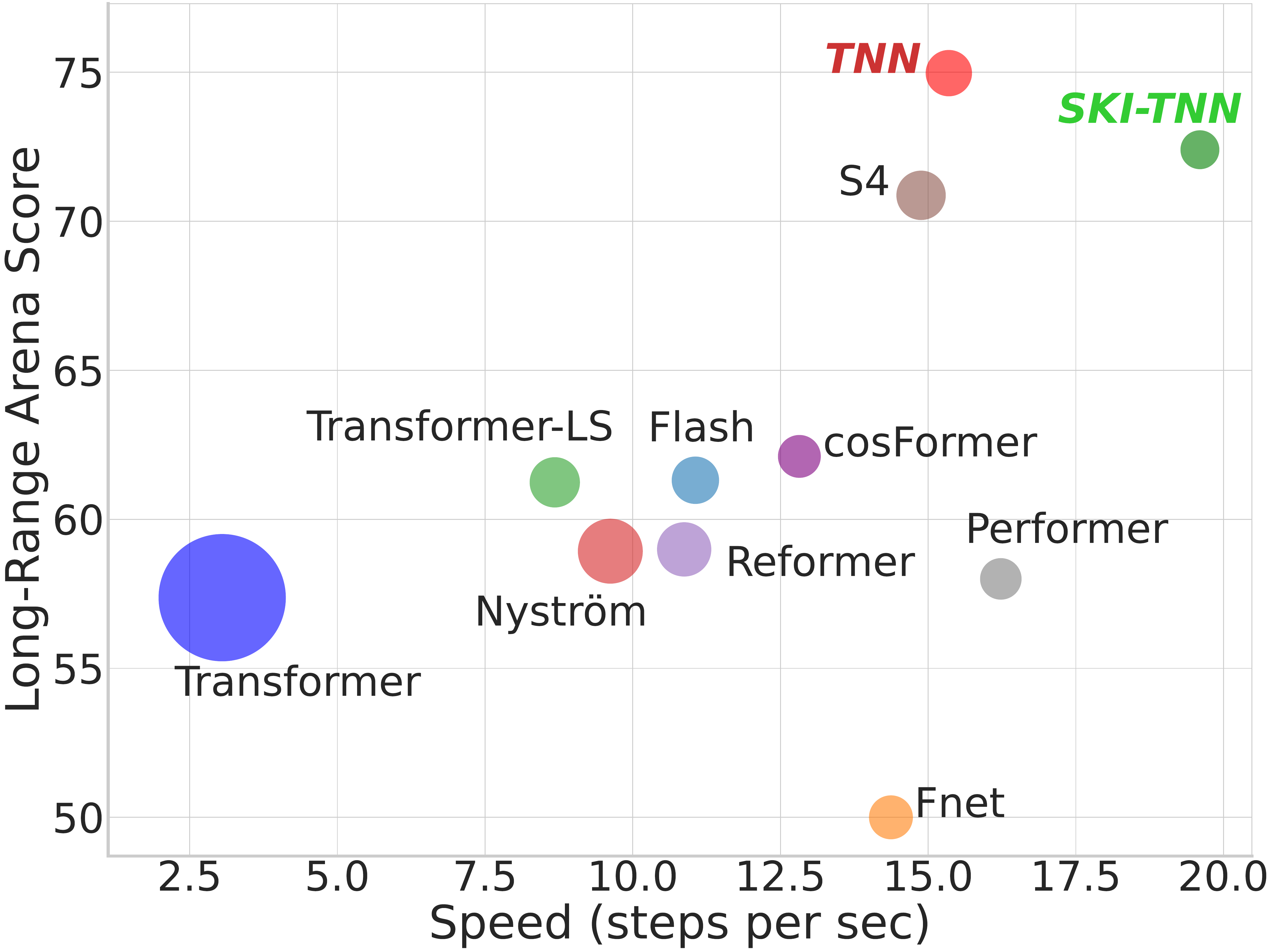}
            \end{subfigure}
            \caption{Long Range Arena (LRA). Our approach, SKI-TNN is faster than all previous approaches while achieving a nearly state of the art LRA score. The size of the bubble corresponds to the memory consumed by the model during training.}
            \label{fig:lra_bubble}
        \end{figure}

        \begin{figure}[h!]
            \centering
            \begin{subfigure}[t]{\linewidth}
                \centering
                \includegraphics[width=0.7\linewidth]{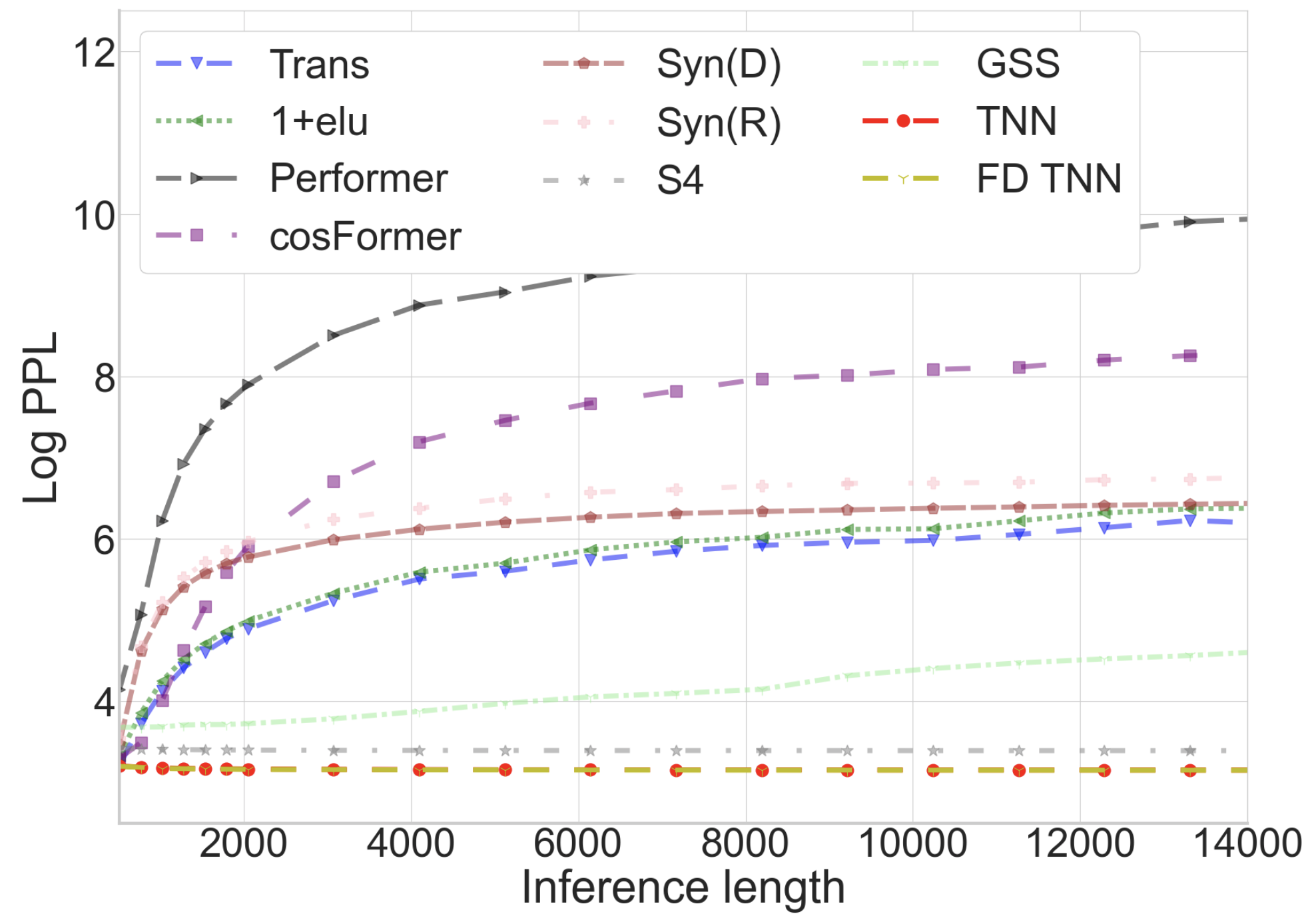}
            \end{subfigure}
            \caption{Wikitext-103 Causal Pretraining. Our approach, FD TNN achieves equivalent perplexity vs inference length to TNN, with ~11-16 percent training speed up on an A100.}
            \label{fig:extrapolation}
        \end{figure}
        \fi

\section{Introduction}
\label{sec:intro}
Sequence modeling is important in natural language processing, where sentences are represented as a sequence of tokens\ifx, and a token's representation can influence and be influenced by surrounding tokens\fi. Successful sequence modeling typically involves token and channel mixing. Token mixing combines representations of different sequence parts, while channel mixing combines the information across different dimensions of embedding vectors used to encode tokens. Transformers \cite{vaswani2017attention} are arguably the most successful technique for sequence modeling, and variants including \cite{hoffmann2022an,clark2020electra} have achieved state of the art performance on natural language tasks. They use self-attention for token mixing and feedforward networks for channel mixing.%Many machine learning applications including speech recognition, language translation, and time-series use sequence modeling. % and their intermediate representations within a neural network

Recently, \cite{qin2023toeplitz} proposed Toeplitz Neural Networks (TNN) using Toeplitz matrices for token mixing\ifx, applying a different linear system per channel\fi. They use a learned neural similarity function, the Relative Positional Encoder (RPE), to form the Toeplitz matrices. Toeplitz matrix vector multiplication can be performed with sub-quadratic complexity using the Fast Fourier Transform (FFT), giving the TNN token mixing layer a total $O(dn\log n)$ computational complexity, where $d$ is the embedding dimension and $n$ is the sequence length. This achieved state of the art predictive performance and nearly state of the art speed for the long range arena (LRA) benchmark \cite{taylong}. They also showed strong performance pre-training wikitext-103 \cite{meritypointer} and on the GLUE benchmark\cite{wangglue}. Despite strong empirical speed performance, TNNs have two fundamental efficiency limitations: 1) super-linear computational complexity 2) many calls to the RPE: for each layer, one call per relative position.

%1) the learned kernels display discontinuous behavior near the main diagonals with otherwise smooth global behavior; 2) as TNNs use FFT, we can achieve a speedup by representing kernels in the frequency domain. For bidirectional models, this motivates a sparse plus low-rank matrix decomposition. We can efficiently apply the kernel's sparse component's action via a small 1D convolution in $O(nk)$ with filter length $k$. For the low rank component, we use the Nystr{\"o}m method for asymmetric kernels \cite{nemtsov2016matrix} and structured kernel interpolation (SKI) \cite{wilson2015kernel} for $O(n)$ complexity. Further, using an inverse time warp, we can extrapolate beyond sequence lengths observed during training and avoid the MLP entirely. However, for autoregressive models, even ``fast'' causal masking \cite{katharopoulos2020transformers} negates the speed and memory benefits from a low-rank decomposition. Thus, we represent the kernel via the real part of its frequency response, and compute the imaginary part via a Hilbert transform to enforce causality. This maintains the $O(n \log n)$ complexity but achieves an absolute speedup. We show speed improvements with only a slight degradation in predictive performance.

In this paper, we interpret the RPE as a non-SPD kernel and note 1) the learned kernels are discontinuous near the main diagonals but otherwise smooth globally; 2) the ReLU RPE learns 1D piecewise linear functions: an MLP is slower than necessary. For bidirectional models, this motivates a sparse plus low-rank decomposition. We apply the sparse component's action via a small 1D convolution. For the low rank component, we replace the RPE MLP with linear interpolation at a set of inducing points and an asymmetric extension of Structured Kernel Interpolation (SKI) \cite{wilson2015kernel} for $O(n)$ complexity. Further, using an inverse time warp, we can extrapolate beyond sequence lengths observed during training. For causal models, even ``fast'' causal masking \cite{katharopoulos2020transformers} negates the speed and memory benefits from SKI. Thus, we instead represent the real part of the kernel's frequency response using the RPE MLP, and evaluate the RPE with finer frequency resolution to extrapolate to longer sequence lengths in the time domain. From the real part, we compute the imaginary part via a Hilbert transform during the forward pass to enforce causality.  In the bidirectional setting, we remove the causality constraint and represent the complex frequency response of the kernel with the RPE MLP. Levels of smoothness in frequency response imply decay rates in the time domain: thus we model the decay bias implicitly. This maintains $O(n \log n)$ complexity but achieves an absolute speedup. Further, it often leads to better predictive performance on LRA tasks.

This paper has three primary contributions: 1) a TNN sparse plus low rank decomposition, extending SKI to TNNs for the low rank part. We replace the RPE MLP with linear interpolation and apply inverse time warping to efficiently train bidirectional TNNs. We provide rigorous error analysis for our asymmetric SKI application; 2) alternatively, for both causal and bidirectional models, we work directly in the frequency domain and use the Hilbert transform to enforce causality in the autoregressive setting. We prove that different activation choices for an MLP modeling the discrete time Fourier transform (DTFT) lead to different decay rates in the original kernel. 3) Empirical results: we demonstrate that our approaches show dramatically improved computational efficiency, setting a new speed state of the art on LRA \cite{tay2022efficient} on the 1d tasks, with strong LRA score. In section \ref{sec:related} we describe related work. In section \ref{sec:modeling} we propose our new modeling approaches. In \ref{sec:theory} we state several theoretical results regarding our modeling approaches. In \ref{sec:experiments} we extend the empirical results of \cite{qin2023toeplitz}, showing our speed gains with minimal prediction deterioration. We conclude in section \ref{sec:conclusion}.% with a discussion of limitations and future work.% we conclude with a discussion of limitations and potential future work.

\section{Related}
\label{sec:related}
The most related papers use Toeplitz matrices for sequence modeling \cite{qin2023toeplitz,luo2021stable,poli2023hyena}. We build off of \cite{qin2023toeplitz} and introduce several techniques to improve on their speed results. \cite{luo2021stable} took a similar approach, but applied Toeplitz matrices to self-attention rather than departing from it. \cite{poli2023hyena} is also similar, using alternating Toeplitz and diagonal matrices as a replacement for self-attention within a Transformer. While we focus on the setting of \cite{qin2023toeplitz} as it was released first, our approach is applicable to \cite{poli2023hyena}.

Also related are kernel based xFormers, particularly those using the Nystr{\"o}m method \cite{nystrom1930praktische,baker1977numerical}. The most related work is \ifx Nystr{\"o}mformer \fi\cite{xiong2021nystromformer}, which adapts a matrix Nystr{\"o}m method for asymmetric matrices \cite{nemtsov2016matrix} to self-attention. We instead adapt this along with SKI \cite{wilson2015kernel} to Toeplitz matrices. \ifx Skyformer \fi\cite{chen2021skyformer} extends \ifx Nystr{\"o}mformer\fi\cite{xiong2021nystromformer} by embedding the self-attention matrix into a larger PSD kernel matrix and approximating the larger matrix instead. Their final approximate matrix has lower spectral error compared to \ifx Nystr{\"o}mformer \fi\cite{xiong2021nystromformer} and higher average validation accuracy on LRA \cite{taylong}. However, their method is slightly slower\ifx than Nystr{\"o}mformer\fi. Also somewhat related are random feature self-attention approximations\cite{peng2021random,choromanski2021rethinking}. These extend \cite{rahimi2007random}, but use different random features that better approximate self-attention than random Fourier or binning features.% For instance, \cite{choromanski2021rethinking} proposed positive random features.

Sparse transformers are also relevant. \cite{child2019generating} proposed using strided and fixed patterns. \cite{brown2020language} alternated between sparse locally banded and dense attention. Finally, \cite{zaheer2020big} proposed combining random attention, window attention and global attention. Our use of a short convolutional filter is most similar to window attention. The space of efficient transformers is huge and there are many models that we haven't covered that may be relevant. \cite{tay2022efficient} provides an excellent survey.

Other successful long sequence approaches include state space models \cite{guefficiently,smith2023simplified,dao2022hungry}, long convolution \cite{romerockconv,fu2023simple}, adding moving averages to gated attention \cite{ma2023mega} and more \cite{khalitov2023chordmixer}.

\section{Modeling Approach}
\label{sec:modeling}
We review Toeplitz neural networks (TNNs) in section \ref{sec:prelims}. We next speed up the TNN's Toeplitz neural operator (TNO). We discuss using Nystr{\"o}m and SKI approaches to bidirectional training in \ref{sec:bidirectional-training}. We discuss frequency based approaches, particularly for causal training in \ref{sec:freq-based-approaches}.

\subsection{Preliminaries: Toeplitz matrices and Toeplitz Neural Networks}\label{sec:prelims}

TNNs \cite{qin2023toeplitz} replace self-attention, which computes the action of self-attention matrices that encode the similarity between both observation values and absolute positions, with the action of Toeplitz matrices that encode similarity only based on \textit{relative} positions. Toeplitz matrices have, for each diagonal, the same entries from left to right. That is, $\T_{ij}=t_{i-j},\T\in \Rbb^{n\times n}$.
    % \[
    %     \T = \left(\begin{matrix}
    %     t_n & t_{n+1} & \ldots & t_{2n-2} & t_{2n-1} \\
    %     t_{n-1} & t_n & \ldots & t_{2n-3} & t_{2n-2} \\
    %     \vdots & \vdots & \ddots & \vdots & \vdots \\
    %     t_{2} & t_{3} & \ldots & t_n & t_{n+1}\\
    %     t_{1} & t_{2} & \ldots & t_{n-1} & t_n
    %     \end{matrix}\right).
    % \]
    Unlike self-attention matrices, which require $O(n^2)$ memory, a Toeplitz matrix has $2n-1$ unique elements and requires $O(n)$ memory. Due to close connections with discrete-time convolution, $\T\x$ can be computed in $O(n\log n)$ time by embedding $\T$ in a circulant matrix and applying FFT.
    % This involves embedding $\T\in \Rbb^{n\times n}$ in a circulant matrix
    % $
    %     \C=\begin{pmatrix}
    %         \T & \C_{12} \\
    %         \C_{21} & \C_{22}
    %     \end{pmatrix}\in \Rbb^{2n\times 2n}
    % $
    % and computing
    % $
    %     \C \begin{pmatrix}
    %          \x  \\
    %          \textbf{0} 
    %     \end{pmatrix}=\begin{pmatrix}
    %          \T\x  \\
    %          \textbf{0} 
    %     \end{pmatrix}
    % $
    % using the FFT.

A TNN \cite{qin2023toeplitz} has multiple sequence modeling blocks, which we show in Figure \ref{fig:tnn} in Appendix \ref{sec:tnn_arch_diagram}. Each block has a Gated Toeplitz Unit (GTU), which does both token and channel mixing, followed by a Gated Linear Unit (GLU) \cite{shazeer2020glu}, which does channel mixing. The core of the GTU is the Toeplitz Neural Operator (TNO), which does token mixing and is the part of the architecture that we modify. %Given a sequence $\X\in \Rbb^{n\times d}$, where $n$ is the sequence length and $d$ is the embedding dimension, it takes as input a set of $2n-1$ relative positions. It then passes them through the relative position encoder (RPE) neural network to get $(2n-1)d$ encodings, one for each relative position and dimension, respectively. This along with a decay bias is used to form $d$ Toeplitz matrices. For each Toeplitz matrix $\T$ associated with a column $\x$ of $\X$, they then compute $\T\x$. This requires $O(dn\log n)$ for $d$ such computations along with $2n-1=O(n)$ calls to the RPE. We next summarize the TNO in more detail with equations.

\begin{figure}
    \centering
    \begin{subfigure}[t]{0.275\linewidth}
        \centering
        \includegraphics[scale=0.6]{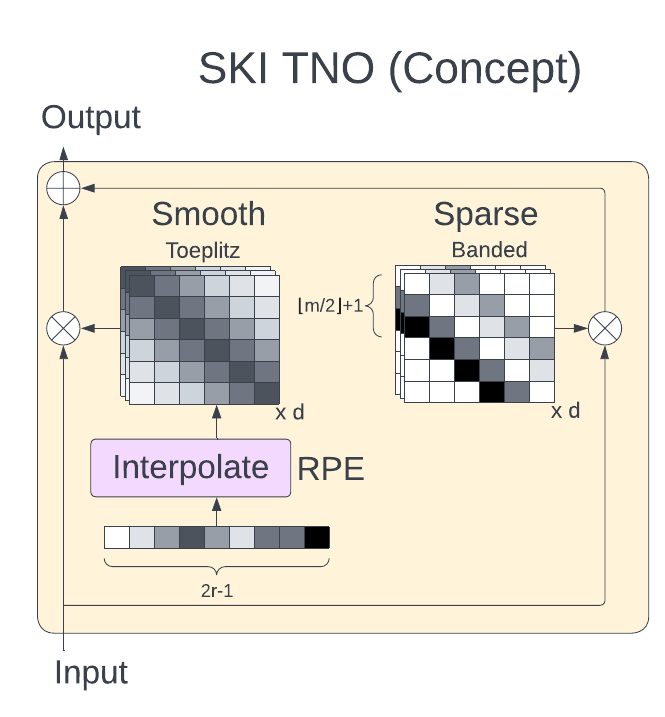}
        \caption{SKI-TNO.}
    \end{subfigure}~
       \begin{subfigure}[t]{0.275\linewidth}
        \centering
        \includegraphics[scale=0.6]{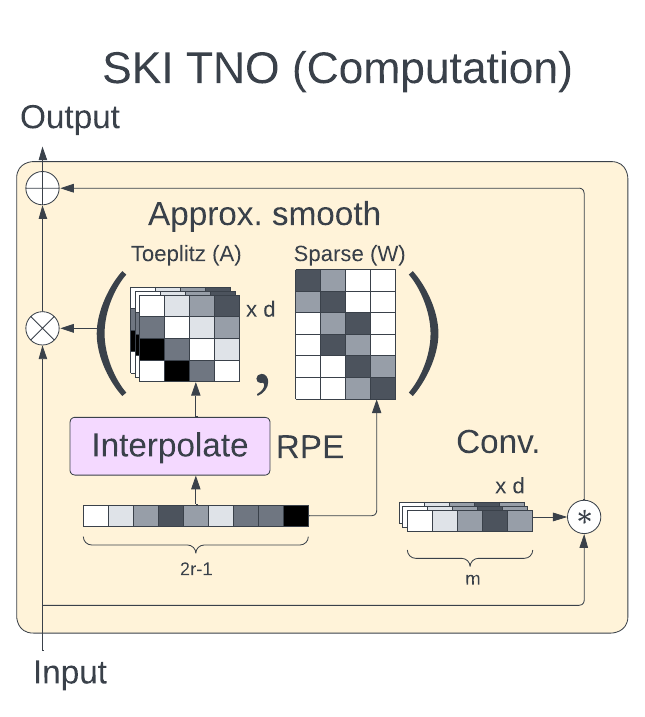}
        \caption{Fast implementation of SKI-TNO.}
    \end{subfigure}~
    \begin{subfigure}[t]{0.2\linewidth}
        \centering
        \includegraphics[scale=0.6]{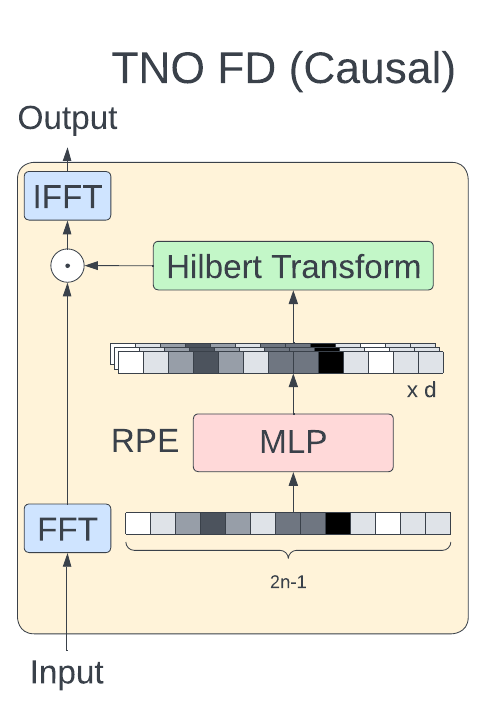}
        \caption{FD-TNO causal.}
    \end{subfigure}~
       \begin{subfigure}[t]{0.2\linewidth}
        \centering
        \includegraphics[scale=0.6]{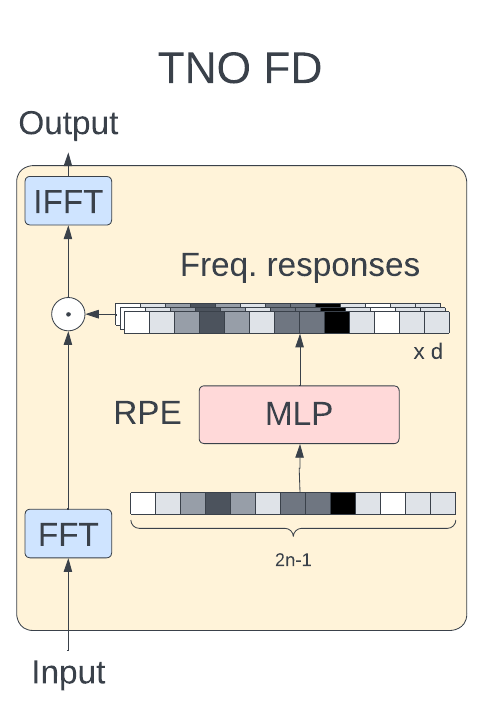}
        \caption{FD-TNO bidirectional.}
    \end{subfigure}%
    \caption{Our SKI-TNO and FD-TNO modifications: (a) We decompose Toeplitz matrices into sums of sparse + smooth components. Additionally, we use interpolation instead of an MLP to learn the RPE. (b) We use a 1D convolution to apply the sparse component and SKI as a low-rank approximation to the smooth component. (c) For the causal case, we use frequency domain RPE with a Hilbert Transform to enforce causality. (d) Our FD-TNO also is competitive in the bidirectional case, with one fewer FFT than TNO.}
    \label{fig:ski_tno}
\end{figure}

We now describe the TNO, shown in Figure \ref{fig:tnn:tno} of Appendix \ref{sec:tnn_arch_diagram}. Given a sequence $\X\in \Rbb^{n\times d}$ of length $n$ and dimension $d$ in discrete time, there are $2n-1$ unique relative positions/times $i-j$ for $i,j=1,\ldots,n$. An $\textrm{RPE}:\mathbb{Z}\rightarrow \Rbb^d$ neural network maps each relative position to a $d$-dimensional embedding. These embeddings are used to construct Toeplitz matrices $\T^l$ for $l=1,\ldots,d$ using
\begin{align*}
    \T^{l}_{ij}&=\lambda^{\vert i-j\vert}\textrm{RPE}_l(i-j).
\end{align*}
$\textrm{RPE}_l(i-j)$ is a learned similarity between positions for dimension $l$, while $\lambda^{\vert i-j\vert}$ with $\lambda\in (0,1)$ is an exponential decay bias penalizing far away tokens to be dissimilar. We can interpret $\T^{l}_{ij}$ as evaluating a stationary non-SPD kernel $k_l(i-j)=\lambda^{\vert i-j\vert}\textrm{RPE}_l(i-j)$. Thus $\T^l$ can be interpreted as a pseudo or generalized Gram matrix. Letting $\x^l$ be the $l$th column of $\X$, the TNO outputs
\begin{align*}
    \textrm{TNO}(\X)&=(\T^1 \x^1\ldots \T^d\x^d)\in \Rbb^{n\times d}
\end{align*}
where each $\T^l\x^l$ is computed via the FFT as described above.

The main costs are the RPE's MLP, the FFT, and the decay bias. We aim to eliminate the MLP and decay bias when possible. In the bidirectional setting, we use SKI to apply the FFT using a much smaller Toeplitz matrix. In a separate model we learn the RPE's frequency response directly. In the bidirectional setting, this allows us to both avoid explicitly modeling the decay bias and use one fewer FFT. In the causal setting, it allows us to avoid explicitly modeling the decay bias.

\subsection{SKI Based Approaches for Bidirectional Training}\label{sec:bidirectional-training}
For a given Toeplitz matrix $\T$, we assume it admits a decomposition that we can approximate with a sparse+low-rank representation, $\T=\T_{\textrm{sparse}}+\T_{\textrm{smooth}}\approx \T_{\textrm{sparse}}+\T_{\textrm{low}}$. Our bidirectional training thus consists of three primary components. The first, the sparse component $\T_{\textrm{sparse}}$ is straightforward. Applying the action $\T_{\textrm{sparse}}\x$ of $\T_{\textrm{sparse}}\in \Rbb^{n\times n}$ with $m$ non-zero diagonals is equivalent to applying a 1D convolution layer with filter size $m$. We then discuss our asymmetric SKI for $\T_{\textrm{low}}$ in section \ref{sec:asymmetric-nystroem}. Finally, we discuss how we handle sequence lengths not observed in training for $\T_{\textrm{low}}$ via an inverse time warp in section \ref{sec:inverse-time-warp}. Algorithm \ref{alg:asymmetric-ski-bidirectional} summarizes our TNO based on these techniques.

\begin{algorithm}
\caption{Sparse Plus Low Rank Bidirectional TNO with Asymmetric SKI}
\begin{algorithmic}
\State \textbf{Given} sequence $\X \in \Rbb^{n\times d}$ with columns $\x^l$
\State \textbf{Hyperparameters} rank $r\ll n$, sparse filter size $m$, interpolation degree $N$, decay parameter $\lambda$
\State \textbf{Compute} inducing points $p_1,\ldots,p_r$ evenly spaced on $[0,n]$
\For{$l=1,\ldots,d$}
    \State Compute $\T_{\textrm{sparse}}^l \x^l$ with a 1D convolutional filter, size $m$.
    \State Let $x(t)=\textrm{sign}(t)\lambda^{\vert t\vert}$.
    \State Form $\A^l\in \Rbb^{r\times r}$ with entries $\A_{ij}^l=k_l(p_i-p_j)=\textrm{RPE}_l(x(p_i-p_j))$
    \State Form $\W^l\in \Rbb^{n\times r}$ degree $N$ polynomial interpolation matrix
    \State Compute $\T_{\textrm{low}}^l\x^l$ with $\T_{\textrm{low}}^l=\W^l \A^l \W^{l\top }$
\EndFor
\State Return $\textrm{TNO}(\X)=(\T_{\textrm{sparse}}^1 \x^1+\T_{\textrm{low}}^1\x^1,\ldots,\T_{\textrm{sparse}}^d \x^d+\T_{\textrm{low}}^d\x^d)$
\end{algorithmic}
\label{alg:asymmetric-ski-bidirectional}
\end{algorithm}

\subsubsection{SKI For Asymmetric Nystr\"{o}m}\label{sec:asymmetric-nystroem}

Given an asymmetric stationary kernel $k:\mathbb{R}\times \mathbb{R}\rightarrow \mathbb{R}$, we wish to approximate the (pseudo) Gram matrix $\T\in\Rbb^{n\times n}$ using a low-rank approximation based on a smaller Gram matrix $\A\in\Rbb^{r\times r}$, with $r\ll n$. In context, $\A$ is formed using relative positions between a set of inducing points $p_1,\ldots,p_r$ instead of the full set $1,\ldots,n$ that is used for $\T$. That is,
\begin{align*}
    \T_{ij}=k(i-j) \qquad \textrm{and} \qquad \A_{ij}=k(p_i-p_j).
\end{align*}
In our case, the inducing points are uniformly spaced. Some submatrices of $\A$ may be submatrices of $\T$ (if inducing points are also observation points). To derive the Nystr\"{o}m approximation, we form an augmented Gram matrix $\K\in\Rbb^{(n+r) \times (n+r)}$ in block form as
\begin{align*}
    \K&=\begin{pmatrix}
         \A & \B \\
         \F & \T
    \end{pmatrix},
\end{align*}
where $\B\in\Rbb^{r \times n}$ and $\F\in\Rbb^{n \times r}$ are respectively the upper right and lower left partitions of the large Gram matrix $\K$. Explicitly,
\begin{align*}
    \B_{ij}=k(p_i-j) \qquad \textrm{and} \qquad \F_{ij}=k(i-p_j).
\end{align*}

Extending \cite{nemtsov2016matrix} to allow singular $\A$,
\begin{align*}
    \widehat{\K}&=\begin{pmatrix}
        \A  \\
        \F 
    \end{pmatrix}\A^{\dagger}\begin{pmatrix}
         \A & \B 
    \end{pmatrix}=\begin{pmatrix}
         \A & \A\A^{\dagger}\B \\
         \F\A^{\dagger}\A & \F \A^{\dagger} \B
    \end{pmatrix}
\end{align*}
where $\A^{\dagger}$ is the Moore-Penrose pseudo-inverse satisfying $\A\A^{\dagger}\A=\A$ (but not necessarily $\A\A^{\dagger}=\I$ as in \cite{nemtsov2016matrix}, which shows up in our different expressions for off-diagonal blocks of $\widehat{\K}$). Following structured kernel interpolation (SKI) \cite{wilson2015kernel}, we approximate $\F$ and $\B$ using interpolation.
Specifically,% \textcolor{red}{and $\A^{\dagger}\A\A^{\dagger}=\A^{\dagger}$ right?} \textcolor{magenta}{precisely, but we don't end up using that so i omitted it}
\begin{align*}
    \F \approx \W \A \qquad \textrm{and} \qquad \B \approx \A \W^\top
\end{align*}
where $\W\in\Rbb^{n\times r}$ is a matrix of sparse interpolation weights with up to two non-zero entries per row for linear interpolation or up to four for cubic. These weights can be computed in closed form from the inducing points $p_i$ and the observation points $i$. Thus we have
\begin{align*}
    \T &\approx \F \A^{\dagger} \B \approx \W \A\A^{\dagger} \A\W^\top = \W\A\W^\top\\
    \Rightarrow \widetilde{\T}&=\W \A \W^\top
\end{align*}
as desired. We can set $\T_{\textrm{low}}=\tilde{\T}$ and compute $\tilde{\T}\x$ by first applying $\W^\top \x$, which is an $O(n)$ operation due to $\W\in \Rbb^{n\times r}$ having sparse rows. Next, we apply $\A(\W^\top \x)$. Since $\A$ is a Toeplitz matrix, this is $O(r\log r)$ as per Section \ref{sec:prelims}. Finally, $\W (\A \W^\top\x)$, the action of $\W$, is again an $O(n)$ operation. Thus computing $\tilde{\T}\x$ is $O(n+r\log r)$ computation. On a GPU, this factorization achieves a speedup from having small $r$ and being able to leverage efficient parallelized matrix multiplication on specialized hardware. However, in PyTorch \cite{paszke19pytorch}, we note that for medium sized matrices up to $n=512$, the time required for data movement in order to perform sparse-dense matrix multiplications can be higher than that of simply performing dense matrix multiplication. This means that in practice, we may instead choose to perform batched dense matrix multiplication, which yields an absolute speedup but a worse asymptotic complexity of $O(nr^2+r\log r)$.%\textcolor{red}{Should $m$ be $r$?} %\textcolor{red}{this may depend on matrix size.} \textcolor{magenta}{good point. I'm pretty sure the matmul itself is slower up to at least 512 (probably 1024 too) due to unoptimized memory access, but the additional overhead of having to reshape makes it really not competitive even up to 2048. TF or jax could be different, but for now I put in a placeholder. we should verify}

\subsubsection{Inverse Time Warp}\label{sec:inverse-time-warp}

TNNs use $k_l(i-j)=\lambda^{\vert i-j\vert}\textrm{RPE}_l(i-j)$, where $\textrm{RPE}_l(i-j)$ is an MLP. There are two issues: 1) the sequential computations required for an MLP are slow, and we only need to evaluate at $2r-1$ points using SKI instead of $2n-1$ to produce the full matrix; 2) extrapolation is used in extending to longer sequence lengths than the MLP was trained on, which is generally less reliable than interpolation.

In Proposition \ref{prop:multiple-piecewise-linear}, we note that an MLP $f:\Rbb\rightarrow \Rbb^d$ with ReLU activations and layer normalization is $d$ piecewise linear functions. As we only need to evaluate at $2r-1$ points, we could let $\textrm{RPE}_l$ be a piecewise linear function with $r$ grid points. However, we still need to handle extrapolation. We use an inverse time warp and let $\textrm{RPE}_l$ linearly interpolate on $[-1,1]$ with the constraint $\textrm{RPE}_l(0)=0$ and define $x(t)=\textrm{sign}(t)\lambda^{\vert t\vert}$ for some $0<\lambda<1$. We then let $k_l(i-j)=\textrm{RPE}_l(x(i-j))$.

% The MLP is expensive and relies on extrapolation, so we would like to remove it by learning $g(t)$ more cheaply and preferably by interpolation. One way is to learn a different function $f(x)$ interpolating on $[-1, 1]$. Since this is a fixed interval, we can parametrize it on a fixed grid of $r$ points and evaluate via linear interpolation. Then, $f(x)$ is also a piecewise linear function. Let's constrain $f(0)=0$.

% Then, we can define an inverted time warp $x(t)=\textrm{sign}(t)\lambda^{|t|}$ for some $0 < \lambda < 1$ and take $g(t)=f(x(t))$. \textcolor{red}{So basically we are mapping (relative) time from $\Rbb$ to the interval $[-1,1]$ and then using a piecewise linear function on that interval, right?} \textcolor{magenta}{yep}

\subsection{Frequency Based Approaches}\label{sec:freq-based-approaches}
\subsubsection{Causal Training}\label{sec:causal-training}

The SKI approach allows training bidirectional TNNs with linear complexity. However, fast causal masking negates SKI's benefits (see Appendix 
 \ref{sec:causal-negates-ski}). Thus we need an alternate causal speedup. We use an MLP in the Fourier domain to avoid an explicit time domain decay bias, and use the Hilbert transform to enforce causality. We now describe how we can learn a causal kernel when working in frequency domain (FD). We first define the discrete Hilbert transform, the key tool for achieving this.
%In section \ref{sec:hilbert-transform-causal}, w
%\begin{figure}
%    \centering
%    \includegraphics[scale=0.8]{figs/tno_fd_causal.pdf}%
%    \includegraphics[scale=0.8]{figs/tno_fd.pdf}
%    \caption{Toeplitz Neural Operator in Frequency Domain}
%    \label{fig:tno_fd}
%\end{figure}

%\subsubsection{Causality via Hilbert Transform}\label{sec:hilbert-transform-causal}

\begin{definition}
    The \textbf{discrete Hilbert transform} of the discrete Fourier transform $\hat{k}$ is given by
    \begin{align*}
        \mathcal{H}\{\hat{k}\}&=\hat{k}*h
    \end{align*}
    where $*$ denotes convolution and
    \begin{align*}
        h[l]&=\begin{cases}
        0\textrm{, $l$ even}\\
        \frac{2}{\pi l}\textrm{, $l$ odd}
        \end{cases}
    \end{align*}
%     \begin{align*}
%         \mathcal{H}\{\hat{k}\}(\omega)&=\textrm{P.V.}(\hat{k}(\cdot)*\frac{1}{\pi \cdot})(\omega)\\
%         %&=\textrm{P.V.}\int_{-\pi}^\pi \hat{k}(\omega-u)\frac{1}{\pi u}du
% t    \end{align*}
    % \begin{align*}
    %     \mathcal{H}\{k\}[n]&=\mathcal{F}^{-1}\{\hat{k}\}*\mathcal{F}^{-1}(-i\textrm{sgn}(\cdot))\\
    %     &=k[n]*\frac{1}{2\pi}\int_{-\pi}^\pi (-i\textrm{sgn}(\omega))\exp(i\omega n)d\omega \\
    %     &=u[n]*h[n]
    % \end{align*}
\end{definition}

The real and imaginary parts of the Fourier transform of a causal function are related to each other through the Hilbert transform. Thus, in order to represent a causal signal, we can model only the real part and compute the corresponding imaginary part. That is, we first estimate an even real function $\hat{k}$ (symmetric about $0$) using an MLP. We then take $\hat{k}_{\textrm{causal}}(\omega) =\hat{k}(\omega)-i\Hcal\{\hat{k}\}(\omega)$.

The inverse Fourier transform $k_{\textrm{causal}}$ of $\hat{k}_{\textrm{causal}}$ will thus be causal. For a discussion of why this ensures causality, see \cite{oppenheim10discrete}. See Algorithm \ref{alg:causal-discrete-hilbert} for TNO pseudocode using this approach. Different choices for the smoothness of the frequency domain MLP will lead to different decay rates in time domain, so that smoothness in frequency domain essentially serves the same purpose as the decay bias in \cite{qin2023toeplitz}. We discuss this theoretically in Section \ref{sec:smoothness-fourier-decay-time}. Note that we also find that working directly in the frequency domain for bidirectional models (without the Hilbert transform) is often competitive with SKI for speed (despite being $O(n\log n)$ instead of $O(n+r\log r)$) due to needing one fewer FFT.

\begin{algorithm}
\caption{Causal TNO via Discrete Hilbert Transform}
\begin{algorithmic}
\State \textbf{Given} sequence $\X \in \Rbb^{n\times d}$ with columns $\x^l$
\State \textbf{Hyperparameters} activation function
\For{$l=1,\ldots,d$}
    \State $\hat{\x}^l \gets \mathcal{F}\{\x^l\}$, where $\mathcal{F}$ is the rFFT.
    \State Compute even real function $\hat{k}^l=\textrm{RPE}_l(\omega)$, $\omega=\frac{m\pi}{n},m=0,\ldots,n$.
    \State Take discrete Hilbert transform $\Hcal\{\hat{k}^l\}$ via the rFFT and irFFT.
    \State Compute $\hat{k}_{\textrm{causal}}^l(\omega)=\hat{k}^l(\omega)-i\Hcal\{\hat{k}^l\}(\omega)$ for $\omega=\frac{m\pi}{n},m=0,\ldots,n$.
    \State $\y^l \gets \mathcal{F}^{-1}\{\hat{k}_{\textrm{causal}}^l \odot \hat{\x}^l\}$, where $\mathcal{F}^{-1}$ is the irFFT and $\odot$ denotes an element-wise product.
\EndFor
\State Return $\textrm{TNO}(\X)=(\y^1, \ldots,\y^d)$
\end{algorithmic}
\label{alg:causal-discrete-hilbert}
\end{algorithm}

\subsubsection{Bidirectional Training with FD TNN}\label{sec:blm-training-fdtnn}

We extend the FD approach to bidirectional training by removing the causality constraint and model the complex frequency response of real valued time domain kernels directly.  To do so we simply double the output width of the RPE and allocate each half for the real and imaginary parts of the kernel frequency responses, while explicitly forcing real-valued responses at  $\omega=0$ and $\pi$.   While increasing the complexity of the RPE slightly, we achieve the speed ups
in Figure \ref{fig:lra_bubble_and_extrapolation} by eliminating the FFTs for the kernels and causality constraint, in addition to the decay bias.

\section{Theory}
\label{sec:theory}
We show in Proposition \ref{prop:multiple-piecewise-linear} that an MLP mapping from scalars with layer norm and ReLU activations is piecewise linear and continuous, suggesting that using an MLP that we only need to evaluate at a small number of points may be overparametrized, justifying the use of interpolated piecewise linear functions. In section \ref{sec:matrix-error-bound} we analyze the spectral norm of the matrix approximation error for SKI. We assume the sparse component is exactly identifiable and bound the error of approximating the smooth term via a low-rank SKI factorization. We leave the problem of relaxing this assumption to future work. In section \ref{sec:smoothness-fourier-decay-time}, we analyze how by using different activations with different smoothness when learning the DTFT of the kernel, we obtain corresponding decay rates for the time domain signal.

\begin{restatable}{proposition}{multiplepiecewiselinear}\label{prop:multiple-piecewise-linear}
A ReLU MLP $f:\Rbb\rightarrow\Rbb^d$ with layer norm and no activation on its output is $d$ piecewise linear continuous functions.
\end{restatable}
\begin{proof}
    See Appendix \ref{sec:appendix-relu-piecewise-linear}.
\end{proof}
\subsection{Matrix Approximation Spectral Norm Error}\label{sec:matrix-error-bound}
% We first show the error for using a sparse banded matrix to approximate $\T$.
% \begin{proposition}\label{prop:sparse-matrix-approx}
% Assume that for some $b>0$,
% \begin{align*}
%     \T_{\textrm{sparse},ij}&=\begin{cases}
%     \lambda^{\vert i-j\vert}\textrm{RPE}(i-j),\vert i-j\vert<b\\
%     0\textrm{ else}
%     \end{cases}
% \end{align*}
% and the RPE is continuous. Then there exists some $B>0$ such that
% \begin{align*}
%     \Vert  \T-\T_{\textrm{sparse}}\Vert_2 &\leq n\lambda^b B
% \end{align*}
% \end{proposition}
% \begin{proof}
% See Appendix \ref{proof:sparse-matrix-approx}.
% \end{proof}
We give our main error bound for our SKI based low rank approximation. Note that this requires that our kernel is $N+1$ times continuously differentiable, while the kernel we use in practice uses a piecewise linear function and is thus non-differentiable. In theory, we would need a smoother kernel, adding additional computation overhead. However, we find that empirical performance is still strong and thus we simply use piecewise linear kernels but include the error bound for completeness. Our results depends on the Nystr{\"o}m error $\E_{nyst}$: its $l^2$ norm is bounded in \cite{nemtsov2016matrix}.

% \textcolor{red}{I `think' that I am implicitly using the assumption that the observation points are bounded by the inducing points here. Probably fine but handling extrapolation like we are doing may require a bit more careful treatment.} \textcolor{magenta}{where are we doing extrapolation? our observation points are also bounded by inducing points.}
\begin{restatable}{theorem}{skierrorbound}\label{thm:ski-error-bound}
Assume that $\A$ is non-singular and $k:[p_1,p_r]\rightarrow \Rbb$ is an $N+1$ times continuously differentiable function, where $p_1$ is the smallest inducing point and $p_r$ is the largest. Let $\T_{r,opt}$ be the optimal rank $r$ approximation to $\T$ and let 
\begin{align*}
    \E_{SKI}&=\W\A\W^\top-\T_{r,opt}\\
\end{align*}
be the difference between the SKI approximation using linear interpolation and the optimal one, while
\begin{align*}
    \E_{nyst}&=\F \A^{-1} \B-\T_{r,opt}
\end{align*}
is the difference between the Nystr{\"o}m approximation and the optimal one. Then
\begin{align*}
\Vert \E_{SKI}\Vert_2 &\leq \sqrt{nr}\max_{p_{n_1}\leq i\leq p_{n_N}}\frac{\vert \psi_N(i)\vert}{(N+1)!}L\left((N+1)\sqrt{n} +\frac{\min(\sigma_1(\F),\sigma_1(\B))}{
    \sigma_{r}(\A)
    }\right) +\Vert \E_{nyst}\Vert_2.
    % \Vert \E_{SKI}\Vert_2 &\leq \Vert \E_{nyst}\Vert_2+\sqrt{nr}\max_{p_{n_1}\leq i\leq p_{n_N}}\frac{\vert \psi_N(i)\vert }{(N+1)!}L\left( (N+1)\sqrt{n} +\frac{\sigma_1(\F)}{
    % \sigma_{s}(\A)
    % }\right)
\end{align*}
where $\psi_N(i)=\prod_{j=1}^N(i-p_{n_j})$ with $p_{n_j}$ being the $j$th closest inducing point to $i$, $L$ is an upper bound on the $N+1$th derivative of $k$, and $\sigma_{i}(\M)$ denotes the $i$th largest singular value of matrix $\M$.% \textcolor{magenta}{[[By symmetry, could we use $min(\sigma_1(\F),\sigma_1(\B))$ to make it seem a tad smaller?]]\textcolor{red}{I think we can. Just need to say we can do the derivation on either side in the proof.}}
\end{restatable}
\begin{proof}
    See Appendix \ref{proof:ski-error-bound}.
\end{proof}
For linear interpolation $\frac{\vert \psi_N(i)\vert }{(N+1)!}\leq \frac{h^2}{8}$, where $h$ is the spacing between two neighboring inducing points. We have considered the sparse component of the Toeplitz matrix to be identifiable and focused on the error of approximating the smooth component. While there are potential approaches to relaxing this assumption \cite{recht10guaranteed,candes10matrix,zhou11godec,mei18SILVar,chandrasekaran11rank,chandrasekaran12latent,zhang18robust}, they must be adapted properly to the Toeplitz setting. Thus, this additional analysis is outside the scope of this paper and a fruitful direction for future work.

% Note that this bound actually implies \textcolor{red}{find citation} that placing inducing points on an equally spaced grid may be sub-optimal for the 2nd term on the right, and a better instance may be letting them be the zeros of Chebyshev polynomials. However, this approach would then require $O(n+m^2)$ computational complexity instead of $O(n+m\log m)$, as the inducing point Gram matrix would no longer be Toeplitz. Given our current predictive performance is satisfactory, we leave applying this to future work for a case where there is a large performance gap between vanilla TNNs and our SKI based approach.

\subsection{Smoothness in Fourier Domain Implies Decay in Time Domain}\label{sec:smoothness-fourier-decay-time}

We now discuss activation function choices when directly learning the discrete time Fourier transform (DTFT) $\hat{k}$ as an MLP. In practice, we sample the DTFT to obtain the actually computable discrete Fourier transform (DFT) by evaluating the MLP with uniform spacing. Different levels of smoothness of the MLP $\hat{k}$ imply different decay rates of the signal $k$. One can think of the choice of activation function as a parametric form for the decay bias. For an MLP, using a GeLU activation implies super-exponential time domain decay. Using SiLU implies super-polynomial time domain decay. For ReLU the signal is square summable. While this subsection focuses on the theoretical relationship between smoothness and decay, in Appendix \ref{appendix:visualizations-smoothness-decay} we show visualizations demonstrating that these relationships are observed in practice. We first define the DTFT and its inverse.

\begin{definition}
    The \textbf{discrete time Fourier transform} \cite{proakis88introduction,oppenheim10discrete} $\hat{k}$ or 
$\mathcal{F}\{k\}$ of $k$ is given by
    \begin{align*}
        \hat{k}(\omega)&\equiv \sum_{m=-\infty}^\infty k[m]\exp(-i\omega m)
    \end{align*}
\end{definition}
\begin{definition}
    The \textbf{inverse discrete time Fourier transform} of the DTFT $\hat{k}$ is given by
    \begin{align*}
        \mathcal{F}^{-1}\{\hat{k}\}[n]&\equiv \frac{1}{2\pi}\int_{-\pi}^\pi \hat{k}(\omega)\exp(i\omega n)d\omega
    \end{align*}
\end{definition}

We now give three theorems relating smoothness of the DTFT to decay of the signal (its inverse).

\begin{restatable}{theorem}{gelumlpdtftsignaldecay}\label{thm:gelu-mlp-dtft-signal-decay}
Using a GeLU MLP for the DTFT $\hat{k}$, for all $a>0$, the signal $k[n]$ will have decay
\begin{align*}
    k[n]&=O(\exp(-an)).
\end{align*}
\end{restatable}
\begin{proof}
    See Appendix \ref{sec:appendix-gelu}.
\end{proof}
\begin{restatable}{theorem}{silumlpdtftsignaldecay}\label{thm:silu-mlp-signal-decay}
        Using a SiLU MLP for the DTFT $\hat{k}$, the signal $k[n]$ will have decay
        \begin{align*}
        \vert k[n]\vert &\leq \frac{1}{ 2\pi\vert   n\vert ^N}\big\Vert \hat{k}^{(N)}\big\Vert_1
    \end{align*}
    for all $n\neq 0,N\in \mathbb{N}$.
\end{restatable}
\begin{proof}
See Appendix \ref{sec:appendix-silu}.   
\end{proof}

\begin{restatable}{theorem}{relumlpdtftsignalvanishes}
Using a ReLU MLP for the DTFT $\hat{k}$ implies $\Vert k\Vert_2<\infty$ (the signal is square summable).
% That is,
% \begin{align*}
%         \sum_{n=-\infty}^n x[n]^2<\infty
% \end{align*}
\end{restatable}
\begin{proof}
Note that $\hat{k}\in L^2[-\pi,\pi]$ since it is continuous. Then apply Parseval's theorem.
%See Appendix \ref{sec:relu-proofs}, where we first prove an auxiliary result followed by the theorem.
\end{proof}

\section{Experiments}
\label{sec:experiments}
We perform experiments in two areas: pre-training a causal language model on Wikitext-103 \cite{meritypointer} and training bidirectional models on Long-Range Arena. We start with the repositories of the TNN paper\footnote{\texttt{https://github.com/OpenNLPLab/Tnn}\ifx\label{tnn_repo}\fi} and use their training and hyper-parameter settings unless indicated otherwise. We use A100 and V100s for training, and a single A100 for timing experiments.

\subsection{Pre-training on Wikitext-103}

In the causal case we aim to predict the next token, conditional on a fixed length sequence of previous tokens. Table \ref{tab:wikitext-103-causal} compares FD-TNN's causal pre-training perplexity \cite{meritypointer} to existing models: it almost exactly matches that of TNNs. Our approach is faster for the same capacity: at sequence length 512 with 6 layer RPEs (as in the TNN paper), FD TNN is 15\% faster than the baseline TNN on a single A100 GPU. When both use a three layer RPE, FD TNN is 10\% faster. We provide some additional details for this experiment as well as for bidirectional pre-training (we see larger speed gains) in Appendix \ref{sec:experimental-details-additional-results}.

%Figure \ref{tab:causal-lm-speed} shows speed comparisons ofe our FD-TNN approach versus vanilla TNNs, where we are consistently faster: \~11-16\%.%Table \ref{tab:activation-function-ablation} shows how different activation functions in our frequency domain MLP affect perplexity. Each activation function in frequency domain implies a corresponding decay rate in time domain.

\begin{table}[ht]
    \centering
    \begin{tabular}{llll}
        Architecture      & PPL (val)       & PPL (test)    & Params (m)                 \\ \hline
        (Attn-based) & & &  \\\hline
        Trans & 24.40   & 24.78& 44.65\\
        LS & 23.56 &24.05  &    47.89  \\
        Flash &   25.92    & 26.70  & 42.17            \\
        $1+$elu         &  27.44 &  28.05  & 44.65\\
        Performer         &   62.50     &  63.16      &      44.65           \\
        Cosformer             &    26.53    &  27.06      &   44.65      \\\hline
        (MLP-based) & & & \\\hline
        Syn(D) & 31.31 & 32.43 &46.75\\
        Syn(R) & 33.68 &34.78 &44.65\\
        gMLP & 28.08 &29.13 &47.83\\\hline
        (SS-based) & & &\\\hline
        S4       &38.34 &39.66 &45.69                \\
        DSS & 39.39& 41.07 &45.73 \\
        GSS & 29.61 & 30.74 &43.84\\\hline
        (TNN-based) & & & \\\hline
        TNN (reproduced, 3 layers)    & 23.98 (23.96) &24.67 (24.61) &48.68 (48.59)       \\
        FD-TNN: Ours, 3 layers          & 23.97 & 24.56 &  48.58\\ \hline
    \end{tabular}
    \caption{\textbf{Performance on Wikitext-103, Causal Language Model}. We reproduce \cite{qin2023toeplitz}'s table except for the bottom two rows corresponding to the baseline TNN and our FD-TNN. For both we use the same RPE config with 3 layers.  We add in parenthesis the baseline TNN results that we reproduced. We have nearly the same perplexity as the baseline TNN. Our approach is faster: at sequence length 512 with a six layer RPE (as in the TNN paper), FD TNN is 15\% faster than the baseline TNN. For a three layer RPE, it is 10\% faster.\ifx Table \ref{tab:causal-lm-speed} shows speed comparisons.\fi}
    \label{tab:wikitext-103-causal}
\end{table}

% \begin{table}[ht]
%     \centering
%     \begin{tabular}{c|c|c}
%     Activation Function & PPL (val) & Average PPL (extrapolation)\\\hline
%         GeLU & & \\
%         SiLU & & \\
%         ReLU & 23.97 &
%     \end{tabular}
%     \caption{\textbf{Ablation for activation functions in FD-TNO}. The activation functions are listed in decreasing order of smoothness. Levels of smoothness in frequency domain correspond to decay rates in time domain. GeLU leads to super-exponential decay, SiLU super-polynomial, and ReLU super-linear.\textcolor{red}{Comment on empirical performance.}}
%     \label{tab:activation-function-ablation}
% \end{table}
% \begin{figure}
%     \centering
%     \includegraphics[width=0.45\linewidth]{figs/alm_extrap_all.png}
%     \caption{Speed for Causal Language Modeling\textcolor{red}{current figure is a placeholder.}}
%     \label{fig:causal-lm-speed}
% \end{figure}
% \begin{table}[]
%     \centering
%     \begin{tabular}{c|c}
%          &  \\
%          & 
%     \end{tabular}
%     \caption{Speed Comparisons for Causal Language Modeling.}
%     \label{tab:causal-lm-speed}
% \end{table}
\subsection{Long-Range Arena}
The Long-Range Arena (LRA) is a benchmark with several long sequence datasets.
The goal is to achieve both high LRA score (predictive performance) and training steps per second. Following \cite{qin2023toeplitz}, we take the TNN architecture and their tuned hyperparameter (HP) configurations\footnote{\texttt{https://github.com/OpenNLPLab/lra}}, simply replacing their TNO module with our SKI-TNO module with $r=64$ and $m=32$. We use $\lambda=0.99$ where they set $\lambda=1$, but otherwise perform \textit{no additional HP tuning} on 1D tasks and use smaller layers $r=32$ and $m=16$ for the 2D tasks. For FD-TNN, we simply use a same-sized RPE for all tasks except a 3-layer RPE for the CIFAR task. \ifx This is in fact to our detriment: \fi We could potentially achieve even higher accuracy with more comprehensive tuning on the 2D tasks or \textit{any} tuning for the 1D tasks. We select the checkpoint with the highest validation accuracy and report the corresponding test accuracy. SKI-TNN achieves similar average accuracy than TNN at lower size, while FD-TNN achieves \textit{higher} accuracy. We suspect that for some of these problems, the square summable signal implied by ReLU in frequency domain is a better parametric form than applying exponential decay bias. We show our results in Table \ref{tab:LRA}.% Rows above our method are taken from TNN.

\ifx
\begin{table}[ht]
    \centering
    \begin{tabular}{l|lllll|l}
        Architecture      & Text       & ListOps    & Retrieval  & Pathfinder & Image      & Avg        \\ \hline
        Transformer       & 61.95      & 38.37      & 80.69      & 65.26      & 40.57      & 57.37      \\
        Nystr\"{o}mformer & 64.83      & 38.51      & 80.52      & 69.48      & 41.30      & 58.93      \\
        Performer         & 64.19      & 38.02      & 80.04      & 66.30      & 41.43      & 58.00      \\
        cosFormer         & 67.70      & 36.50      & 83.15      & $71.96^*$  & 51.23      & 62.11      \\
        Flash             & 64.10      & 38.70      & $86.10^*$  & 70.25      & 47.40      & 61.31      \\
        S4                & {\ul 85.92}& \bf{50.60} & 67.30      & {\ul 72.44}& \bf{78.07} & $70.87^*$  \\
        TNN               & \bf{86.39} & {\ul 47.33}& \bf{89.40} & \bf{73.89} & {\ul 77.84}& \bf{74.97} \\ \hline
        SKI-TNN           & $83.19^*$  & $45.31^*$  & {\ul 88.73}& 68.30      & $76.46^*$  & {\ul 72.40}\\ \hline
    \end{tabular}
    \caption{\textbf{Performance on Long Range Arena}. We take the table from \cite{qin2023toeplitz} except for the bottom row corresponding to our proposed SKI-TNN. We \textbf{bold}, {\ul underline}, and asterisk$^*$ the \textbf{best}, {\ul second}, and third$^*$ highest accuracy for each task, respectively. Our proposed SKI-TNN achieves the second best overall performance with \textit{no additional hyperparameter tuning}.}
    \label{tab:LRA}
\end{table}
\fi
\begin{table}[ht]
    \centering
    \begin{tabular}{l|lllll|l}
        Architecture      & Text       & ListOps    & Retrieval  & Pathfinder & Image      & Avg        \\ \hline
        TNN               & {\bf 86.39} & {\ul 47.33} & {\ul 89.40} & {\bf 73.89} & {\ul 77.84} & {\ul 74.97} \\ \hline
        SKI-TNN           & 83.19  & 45.31  & 88.73 & 68.30      & 76.46  & 72.40\\
        FD-TNN           & {\ul 85.00}  & {\bf 55.21}  & {\bf 90.26} & {\ul 69.45}      & {\bf 84.12}  & {\bf 76.81}\\\hline
    \end{tabular}
    \caption{\textbf{Performance on Long Range Arena}. We reproduce experiments and train our proposed variants using tuned hyperparameters from \cite{qin2023toeplitz}. We {\bf bold} the best and {\ul underline} the second in each task. Our proposed SKI-TNN and FD-TNN achieve similar overall performance with \textit{no additional hyperparameter tuning} on 1D LRA tasks and a minimal amount of tuning on 2D tasks.}
    \label{tab:LRA}
\end{table}

We additionally perform timing and memory profiling tests on a single 1x A100 instance, keeping the per-GPU batch size constant as in the training runs. In Figure \ref{fig:lra_bubble}, we plot for each 1D task the percentage of TNN accuracy achieved vs the percentage speedup relative to TNN, with the size of the marker corresponding to the peak memory usage measured. We highlight the 1D tasks because they required no tuning, and they represent the longest sequences at lengths ranging from $1024$ to $4096$, whereas the 2D tasks are treated as separate 1D sequences in each dimension, so that a $32\times 32$ image is seen as alternating length $32$ sequences. We note that because the effective sequence lengths are shorter, there is less benefit from using our methods over the baseline TNN.

\section{Conclusion}
\label{sec:conclusion}
In this paper, we note that \cite{qin2023toeplitz}'s Toeplitz neural networks essentially apply the action of a generalized Gram matrix (the Toeplitz matrix) for an asymmetric kernel (the RPE times decay bias) as their main computationally expensive operation. The visualized learned Gram matrices motivate a sparse and low rank decomposition. We thus propose two different approaches to improve efficiency. In the bidirectional setting, we extend SKI to the asymmetric setting and use linear interpolation over a small set of inducing points to avoid the MLP entirely, while using an inverse time warp to handle extrapolation to time points not observed during training. This approach reduces the mathematical complexity from $O(n\log n)$ to $O(n+r\log r)$, where $r$ is the number of inducing points. However in practice, we do not actually use $O(n+r\log r)$ code due to a reshape required for sparse tensors leading to them actually being \textit{slower} than dense tensors. Thus we actually use $O(nr^2+r\log r)$ in code: still much faster than baseline TNN for small $r$. For causal training, as causal masking negates SKI's benefits, we instead eliminate the explicit decay bias. We do this by working directly in the frequency domain, enforcing causality via the Hilbert transform and enforcing decay in time domain via smoothness. For the bidirectional case, we eliminate the FFT applied to the kernels. While this maintains $O(n\log n)$ computational complexity, it leads to a substantial speedup in practice and beats TNNs on LRA score.% Overall, we outperform TNNs on speed with minimal predictive performance degradation. %Finally, regarding societal impact, one potential positive societal impact is that lower memory and compute requirement can lead to shorter training times with a smaller number of GPUs required to achieve the same predictive performance, thus lowering the environmental footprint of language models.

\section{Acknowledgments}
\label{sec:conclusion}
This projects was funded by Luminous Computing. We thank Yiran Zhong and the team from \cite{qin2023toeplitz} for helpful discussions about and updates to their code base. We also thank Tri Dao and Michael Poli for helpful questions and comments. Finally, we thank David Scott for going over the math of the paper.

\bibliography{sample}

\bibliographystyle{plainnat}

\appendix
\newpage
\addcontentsline{toc}{section}{Appendix}
\part{Appendix} % Start the appendix part
\parttoc % Insert the appendix TOC

\section{Toeplitz Neural Network Architecture Diagrams}\label{sec:tnn_arch_diagram}

\begin{figure}[hb]
    \centering
    \begin{subfigure}[t]{0.3\linewidth}
        \centering
        \includegraphics[scale=0.6]{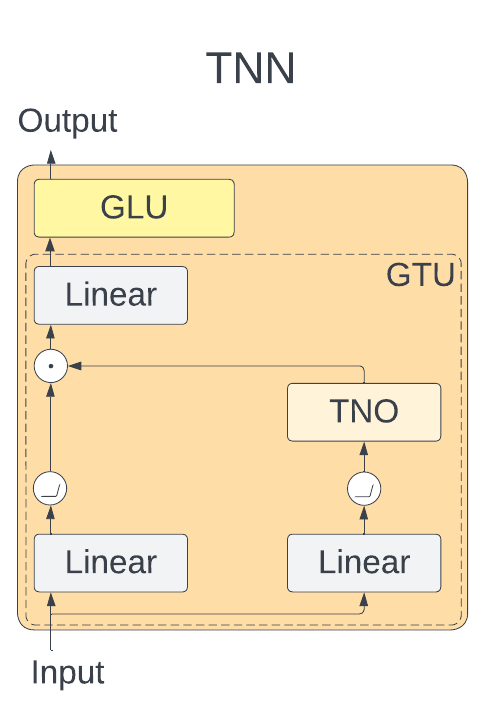}
        \caption{TNN architecture.\label{fig:tnn:tnn}}
    \end{subfigure}~
    \begin{subfigure}[t]{0.3\linewidth}
        \centering
        \includegraphics[scale=0.6]{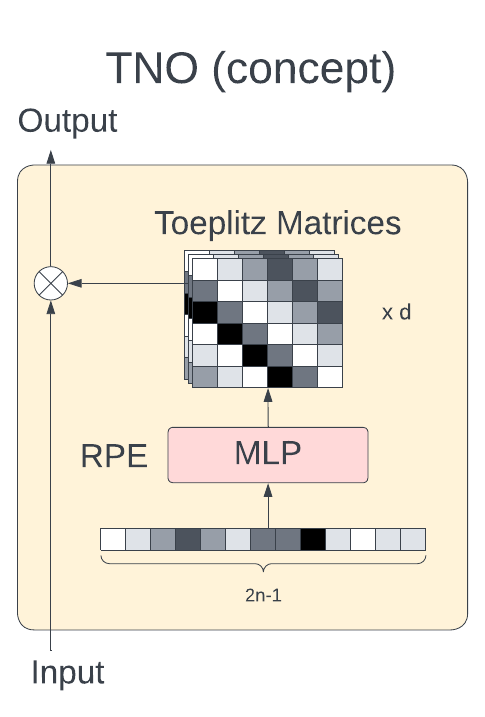}
        \caption{TNO module.\label{fig:tnn:tno}}
    \end{subfigure}~
    \begin{subfigure}[t]{0.3\linewidth}
        \centering
        \includegraphics[scale=0.6]{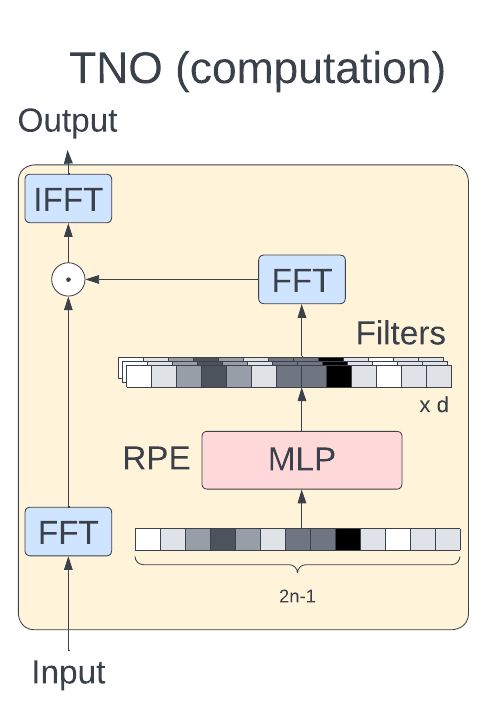}
        \caption{Fast computations to implement TNO.}
    \end{subfigure}
    \caption{Toeplitz Neural Network and Toeplitz Neural Operators: (a) The overall architecture of a TNN layer \cite{qin2023toeplitz}. (b) Conceptually, the TNO multiplies each channel of the input by a different Toeplitz matrix. (c) Computationally, the TNO uses FFT's for speed.}
    \label{fig:tnn}
\end{figure}

\section{Causal Masking negates SKI's benefits}\label{sec:causal-negates-ski}
We now show how requiring causal masking for SKI negates its computational benefits on popular hardware accelerators that optimize parallelized matrix multiplication, such as GPUs. Thus, we will need an alternative approach.

First, let's examine the algorithm from \cite{katharopoulos2020transformers}. 
Let $\x'=\T\x$, the subscripted $\w_i\in \Rbb^{r}$ denote the $i$-th row of $\W$ taken as a column vector, and the subscripted square bracketed $[\W]_i$ denote taking the $i$-th row as a column. That is,
    \begin{align*}
        \x' &=\begin{pmatrix}x'_1 & \ldots & x'_n\end{pmatrix}^\top \qquad \W=\begin{pmatrix}\w_1 & \ldots & \w_n\end{pmatrix}^\top \\
        \x &=\begin{pmatrix}x_1 & \ldots & x_n\end{pmatrix}^\top \qquad
        [\W]_i = \w_i.
    \end{align*}
Then
\begin{align*}
    x_i'&=\sum_{j=1}^i \w_{i}^\top \A \w_{j} x_j
\end{align*}
Let us define intermediate sums and resulting recursions,
\begin{align*}
    &\begin{aligned}
        \s_{i}&\overset{\Delta}{=}\sum_{j=1}^i  \w_{j} x_j \in \Rbb^{r} \\
        \Rightarrow \s_{i+1}&=\s_{i}+\w_{i+1} x_{i+1}
    \end{aligned} &
    \begin{aligned}
        \s'_{i}&\overset{\Delta}{=}\sum_{j=1}^i  \A\w_{j} x_j \in \Rbb^{r} \\
        \Rightarrow \s'_{i+1}&=\s_{i}+\A\w_{i+1} x_{i+1}
    \end{aligned}
\end{align*}
so that
\begin{align*}
    x_i'&=\w_{i}^\top\s_i'=\w_{i}^\top \A \s_i=[\W\A]_i^\top \s_i.
\end{align*}
While we \textit{want} to apply the action of $\A$ to $\W^\top \x \in \Rbb^{r}$ once, which takes $O(r\log r)$. Instead, we \textit{have} to compute one of: (a) $\A\s_i$ $\forall i=1,\ldots,n$; (b) $\W\A$; or (c) $\A\W^\top$; all of which take at least $O(nr)$. However, that is not even the largest practical loss. Instead, it is the fact that both cumulative sums $\s_i$ and $\s'_i$ are sequential in nature to compute efficiently (it \textit{is} possible to parallelize the computation with $O(n^2 r)$ memory complexity, also defeating the purpose of this exercise). We found that the sequential nature of the cumulative sum makes it slower than the baseline TNN with FFTs in practice for moderate sequence lengths of at least up to $2048$ on current GPUs (NVidia V100, A10, A100). Thus, we need to find an alternate approach for the causal setting. %\textcolor{red}{we might need to show a small amount of empirical evidence of this, just a simple case. Happy to help do this if you want.}

\section{Proofs Related to Proposition \ref{prop:multiple-piecewise-linear}}\label{sec:appendix-relu-piecewise-linear}

We first introduce two auxiliary lemmas, and then prove our main result, which follows immediately from the auxiliary lemmas.

\begin{lemma}\label{lemma:relu-mlp-piecewise-linear}
A ReLU MLP $f:\Rbb\rightarrow\Rbb$ with no activation on its output is piecewise linear continuous.
\end{lemma}
\begin{proof}
Each pre-activation node is a linear combination of piecewise linear continuous functions, and is thus piecewise linear continuous. Each activation applies ReLU, which is piecewise linear and the composition of piecewise linear continuous functions is also piecewise linear continuous. The output is a pre-activation and is thus piecewise linear continuous.
\end{proof}

\begin{lemma}\label{lemma:layer-norm-preserves-piecewise}
    Adding layer normalization to a ReLU MLP $f:\Rbb\rightarrow\Rbb$ preserves piecewise linearity.
\end{lemma}
\begin{proof}
Layer normalization applies the same affine transformation to each node in a layer. Since an affine transformation of a piecewise linear continuous function is still piecewise linear continuous, adding layer normalization to an MLP preserves piecewise linear continuity.
\end{proof}
\multiplepiecewiselinear*
\begin{proof}
    Follows immediately from Lemmas \ref{lemma:relu-mlp-piecewise-linear} and \ref{lemma:layer-norm-preserves-piecewise}.
\end{proof}

\section{Proofs for Matrix Approximation Error Spectral Norm}

\subsection{Proof of Theorem \ref{thm:ski-error-bound}}\label{proof:ski-error-bound}
\skierrorbound*
\begin{proof}
We first decompose the difference between the SKI approximation and the optimal rank $r$ approximation into the sum of two terms: the difference between the SKI and the Nystr{\"o}m approximations, and the difference between the Nystr{\"o}m and optimal rank $r$ approximations.
\begin{align*}
    \E_{SKI}&=\W\A\W^\top-\T_{r,opt}\\
    &=\W\A\W^\top-\F\A^{-1}\B+\F\A^{-1}\B-\T_{r,opt}\\
    &=\W\A\W^\top-\F\A^{-1}\B+\E_{nyst}
\end{align*}
so that
\begin{align*}
    \Vert \E_{SKI}\Vert_2 &\leq \Vert \W\A\W^\top-\F\A^{-1}\B\Vert_2 +\Vert \E_{nyst}\Vert_2
\end{align*}

%\textcolor{red}{We could potentially try to look at $WA_MW^\top-F_{lg}A_{lg}^{-1}B_{lg}$ directly, but that would be more difficult.} \textcolor{magenta}{what about using $F_{lg}A_{lg}^{-1}B_{lg} = X_B Y_B$ and $A_M = X_A Y_A$? Then we basically only need to consider $WX_A-X_B$ and $Y_B-Y_A W^\top$, with no inverses of $A$. Granted, I don't know exactly what $X$ and $Y$ are, but at least they're not inverted.}
% Then

We need to bound $\Vert \W\A_M\W^\top-\F\A^{-1}\B\Vert_2$, the operator norm of the difference between the SKI and the Nystr{\"o}m approximations.
\begin{align}
    &\Vert \W \A\A^{-1}\A\W^\top-\F\A^{-1}\B\Vert_2 \nonumber\\
    &\qquad=\Vert \W\A\A^{-1}\A\W^\top-\F\A^{-1}\A\W^\top+\F\A^{-1}\A\W^\top-\F\A^{-1}\B\Vert_2\nonumber\\
    &\qquad\leq \Vert \W\A-\F\Vert_2 \Vert \W^\top \Vert_2+\Vert \F \A^{-1}\Vert_2 \Vert \A \W^\top -\B\Vert_2\nonumber\\
    &\qquad\leq \sigma_1(\W)\Vert \W\A-\F\Vert_2 +\frac{\sigma_1(\F)}{
    \sigma_{r}(\A)
    } \Vert \A \W^\top -\B\Vert_2.\label{eqn:initial-ski-bound}
\end{align}

% We then need an error analysis for different interpolation schemes of $\Vert WA_M-F_M\Vert_2$. Given we are using underlying functions $k$, it may depend on the smoothness of $k$.

% \textcolor{red}{My initial thought: bound $\Vert WA_M-F_M\Vert_F$ and use that to bound $\Vert WA_M-F_M\Vert_2$. Might be too loose to be useful but we'll see.}
The first term describes the error due to approximation of $\F$, the left Nystr{\"o}m factor, while the second term describes the error due to approximation of $\B$, the right one. We can use standard interpolation results to bound $\Vert \W\A-\F\Vert_2$ and $\Vert \A \W^\top -\B\Vert_2$. Recall that the left Nystr{\"o}m factor and inducing Gram matrix have terms
\begin{align*}
    \F_{ij}&=k(i,p_j)\\
    \A_{ij}&=k(p_i,p_j),
\end{align*}
so that $(\W\A)_{ij}=\tilde{k}(i,p_j)$ approximates $\F_{ij}=k(i,p_j)$ using interpolation. For linear interpolation this is
\begin{align*}
    \tilde{k}(i,p_j)= w_i k(p_A,p_j)+(1-w_i)k(p_B,p_j).
\end{align*}
where $p_A,p_B$ are the two closest inducing points to $i$. More generally with polynomial interpolation of degree $N$ we use $p_{n_1},\ldots,p_{n_N}$ to denote the $N$ closest inducing points to $i$. Using the Lagrange error formula, polynomial interpolation has the following error bound \cite{math563polynomial}
\begin{align*}
    \vert \tilde{k}(i,p_j)-k(i,p_j)\vert&\leq \left \vert \frac{\psi_N(i)}{(N+1)!}\right\vert\max_{p_{n_1}\leq x\leq p_{n_N}}\left\vert\frac{\partial^{N+1} }{\partial x^{N+1}}k(x,p_j)\right\vert
\end{align*}
where $\psi_N(i)=\prod_{j=1}^N(i-p_{n_j})$. As an example, for linear interpolation this gives
\begin{align*}
    \vert \tilde{k}(i,p_j)-k(i,p_j)\vert&\leq \left\vert\frac{(i-p_A)(i-P_B)}{2}\right\vert\max_{p_A\leq x\leq p_B}\left\vert\frac{\partial^2 }{\partial x^2}k(x,p_j)\right\vert\\
    &\leq \frac{h^2}{8}\max_{p_A\leq x\leq p_B}\left\vert\frac{\partial^2 }{\partial x^2}k(x,p_j)\right\vert,
\end{align*}
where $h=p_B-p_A$ is the distance between any two neighboring inducing points. Note that we assumed the $N+1$th partial is continuous and since we are interested in $k$ on a compact domain, the $N+1$th partial is bounded, say by $L$. Thus,
\begin{align*}
    \vert \tilde{k}(i,p_j)-k(i,p_j)\vert&\leq \left \vert \frac{\psi_N(i)}{(N+1)!}\right\vert L\\
    \Rightarrow(\tilde{k}(i,p_j)-k(i,p_j))^2&\leq \left ( \frac{\psi_N(i)}{(N+1)!}\right)^2L^2
\end{align*}
and thus we can bound the error in the Frobenius norm of the left factor's SKI approximation as
\begin{align*}
    \Vert \W\A-\F\Vert_F^2&\leq nr \max_{p_{n_1}\leq i\leq p_{n_N}}\left ( \frac{\psi_N(i)}{(N+1)!}\right)^2L^2\\
    \Rightarrow\Vert \W\A-\F\Vert_F&\leq \sqrt{nr}\max_{p_{n_1}\leq i\leq p_{n_N}}\frac{\vert \psi_N(i)\vert}{(N+1)!} L.
\end{align*}
This implies an operator norm bound
\begin{align*}
    \Vert \W\A-\F\Vert_2 &\leq \Vert \W\A-\F\Vert_F\\
    &\leq \sqrt{nr}\max_{p_{n_1}\leq i\leq p_{n_N}}\frac{\vert \psi_N(i)\vert}{(N+1)!} L.
\end{align*}
% Now recall that the right factor has elements
% \begin{align*}
%     \B_{ij}&=k(p_i,j).
% \end{align*}
% Now again using linear interpolation,
% \begin{align*}
%     \tilde{k}(p_i,j)&=w_j k(p_i,p_A)+(1-w_j)k(p_i,p_B).
% \end{align*}
% We then again have the interpolation error
% \begin{align*}
%     ( \tilde{k}(p_i,j)-k(p_i,j))^2 &\leq \frac{h^4}{64}L^2.
% \end{align*}
The right factor approximation $\Vert \A \W^\top -\B\Vert_2$ has the same bound. Plugging into Eqn. \ref{eqn:initial-ski-bound}, we have
\begin{align*}
    \Vert \W \A\A^{-1}\A\W^\top-\F\A^{-1}\B\Vert_2&\leq \sqrt{nr}\max_{p_{n_1}\leq i\leq p_{n_N}}\frac{\vert \psi_N(i)\vert}{(N+1)!}L\left(\sigma_1(\W) +\frac{\sigma_1(\F)}{
    \sigma_{s}(\A)
    }\right)
\end{align*}
which gives
\begin{align*}
    \Vert \E_{SKI}\Vert_2 &\leq \sqrt{nr}\max_{p_{n_1}\leq i\leq p_{n_N}}\frac{\vert \psi_N(i)\vert}{(N+1)!}L\left(\sigma_1(\W) +\frac{\sigma_1(\F)}{
    \sigma_{r}(\A)
    }\right) +\Vert \E_{nyst}\Vert_2.
\end{align*}
Now recall that
\begin{align*}
    \sigma_1(\W)&=\Vert \W\Vert_2\\
    &\leq \sqrt{n}\Vert \W\Vert_\infty\\
    &\leq (N+1)\sqrt{n}
    %&\leq  \Vert \W\Vert_1\\
    %&=\max_{1\leq j\leq n}\sum_{i=1}^r
\end{align*}
since $\W$ has at most $N+1$ non-zero entries in each row \ifx\textcolor{magenta}{feels like we should be able to bound this tighter (maximum absolute row sum)}\fi, so that
\begin{align*}
    \Vert \E_{SKI}\Vert_2 &\leq \sqrt{nr}\max_{p_{n_1}\leq i\leq p_{n_N}}\frac{\vert \psi_N(i)\vert}{(N+1)!}L\left((N+1)\sqrt{n} +\frac{\sigma_1(\F)}{
    \sigma_{r}(\A)
    }\right) +\Vert \E_{nyst}\Vert_2.
\end{align*}
Note that we could have alternatively expanded Eqn. \ref{eqn:initial-ski-bound} using terms based on $\B$ instead of $\F$. This gives
\begin{align}
    &\Vert \W \A\A^{-1}\A\W^\top-\F\A^{-1}\B\Vert_2 \nonumber\\
    &\qquad=\Vert \W\A\A^{-1}\A\W^\top-\W \A \A^{-1}\B+\W\A\A^{-1}\B-\F\A^{-1}\B\Vert_2\nonumber\\
    &\leq \Vert \W\Vert_2 \Vert \A\W^\top  -\B\Vert_2+\Vert \W\A-\F\Vert_2\Vert \A^{-1}\B\Vert_2\nonumber\\
    &\leq \sigma_1(\W)\Vert \A\W^\top  -\B\Vert_2 +\frac{\sigma_1(\B)}{
    \sigma_{r}(\A)
    } \Vert \W\A-\F\Vert_2.\label{eqn:initial-ski-bound-alternate}.
\end{align}
Using Eqn. \ref{eqn:initial-ski-bound-alternate} instead of Eqn. \ref{eqn:initial-ski-bound} and taking the min of both results leads to a bound of
\begin{align*}
    \Vert \E_{SKI}\Vert_2 &\leq \sqrt{nr}\max_{p_{n_1}\leq i\leq p_{n_N}}\frac{\vert \psi_N(i)\vert}{(N+1)!}L\left((N+1)\sqrt{n} +\frac{\min(\sigma_1(\F),\sigma_1(\B))}{
    \sigma_{r}(\A)
    }\right) +\Vert \E_{nyst}\Vert_2.
\end{align*}

\end{proof}
% and
% \begin{align*}
%     \Vert \W\Vert_1 &\leq r
% \end{align*}

% Note that
% \begin{align*}
%     \Vert y\Vert_\infty &\leq \Vert y\Vert_2\leq \sqrt{n}\Vert y\Vert_\infty 
% \end{align*}
% thus
% \begin{align*}
%     \Vert Ay\Vert_2 &\leq \sqrt{n}\Vert Ay\Vert_\infty\\
%     &\leq \sqrt{n}\Vert A\Vert_\infty \Vert y\Vert_\infty\\
%     &\leq \sqrt{n}\Vert A\Vert_\infty \Vert y\Vert_2\\
%     \Vert A\Vert_2 &\leq \sqrt{n}\Vert A\Vert_\infty 
% \end{align*}

\section{Smoothness and Decay}
\subsection{GeLU: Proofs Related to Theorem \ref{thm:gelu-mlp-dtft-signal-decay}}\label{sec:appendix-gelu}
We analyze how modeling the DTFT with a GeLU MLP affects smoothness, the strongest form being an \textit{entire} function, which is complex differentiable everywhere. We then analyze what this implies for the signal. We first recap three basic definitions from complex analysis. In Lemmas \ref{lemma:complex-extension-gelu-activation} and \ref{lemma:gelu-mlp-layer-norm-entire}, we show GeLU MLPs are entire. In \ref{prop:dtft-entire-signal-decay} we show that if a DTFT is entire then the signal will decay at faster than any exponential rate. Finally in Theorem \ref{thm:gelu-mlp-dtft-signal-decay}, we show that modeling the DTFT with a GeLU MLP implies that the signal will decay faster than any exponential rate.
\begin{definition}
    The \textbf{complex derivative} of $f:\mathbb{C}\rightarrow \mathbb{C}$ at $z_0\in \mathbb{C}$ is defined as
    \begin{align*}
        f'(z_0)&=\lim_{z\rightarrow z_0}\frac{f(z)-f(z_0)}{z-z_0}.
    \end{align*}
\end{definition}
\begin{definition}
    A function $f:\mathbb{C}\rightarrow \mathbb{C}$ is \textbf{holomorphic} at $z_0 \in \mathbb{C}$ if it is differentiable on a neighborhood of $z_0$.
\end{definition}
\begin{definition}
    A function is \textbf{entire} if it is holomorphic on $\mathbb{C}$.
\end{definition}

\begin{lemma}\label{lemma:complex-extension-gelu-activation}
    The complex extension of the GeLU activation function is entire.
\end{lemma}
\begin{proof}
The GeLU activation function is $x \Phi(x)$, where $\Phi(x)$ is the standard normal CDF. The complex extension is thus $z\Phi(z)$. Recall that
\begin{align*}
    \Phi(z)&=\frac{1+\textrm{Erf}(z/\sqrt{2})}{2}
\end{align*}
where $\textrm{Erf}$ is the error function. Clearly $z/\sqrt{2}$ is holomorphic on $\mathbb{C}$. It is well known that $\textrm{Erf}$ is holomorphic on $\mathbb{C}$ (see \cite{203920} for proof) and compositions of holomorphic functions are holomorphic. Thus $\Phi(z)$ is holomorphic. Finally, the product of holomorphic functions is holomorphic, so that $z\Phi(z)$ is. Since all of this was holomorphic on $\mathbb{C}$, the complex extension of the GeLU activation function is entire.
\end{proof}

\begin{lemma}\label{lemma:gelu-mlp-layer-norm-entire}
    Each output node of a GeLU MLP with layer norm is an entire function.
\end{lemma}
\begin{proof}
Linear combinations of holomorphic functions are holomorphic, as are compositions. Pre-activations are linear combinations and activations are compositions. The layer-norms are affine transformations, which are also holomorphic. Thus each output node is an entire function.
\end{proof}

\begin{proposition}\label{prop:dtft-entire-signal-decay}
    If the DTFT is entire then 
    \begin{align*}
        k[n]&=O(\exp(-an))
    \end{align*}
    for \textit{all} $a>0$.
    %the signal has exponential decay.\textcolor{red}{make more formal.}
\end{proposition}
\begin{proof}
Let's consider the Fourier series of $\hat{k}(-\omega)$, which is also entire. Its $n$th coefficient is given by
\begin{align*}
    c_n&=\frac{1}{2\pi}\int_{-\pi}^\pi \hat{k}(-\omega)\exp(-\omega i n)d\omega.
\end{align*}
Let $u=-\omega$; then $du=-d\omega$ and

\begin{align*}
    c_n&=-\frac{1}{2\pi}\int_{-\pi}^\pi \hat{k}(u)\exp(uin)du \\
    &= -k[n].
\end{align*}
Now, Fourier series coefficients for analytic functions in a strip $[-a,a]$ decay as $O(\exp(-an))$.
\end{proof}

\ifx
Note that we can decompose $\hat{k}$ into its even and odd parts (the even part being purely real and the odd being purely imaginary), which are also entire
\begin{align*}
    \hat{k}(\omega)&=\hat{k}_e(\omega)+\hat{k}_o(\omega).
\end{align*}

Then the signal is
\begin{align*}
    k[n]&=\frac{1}{2\pi}\int_{-\pi}^\pi [\hat{k}_e+\hat{k}_o](\omega)\exp(\omega in)d\omega.
\end{align*}
Letting $c_{n,e}$ be the $n$th Fourier coefficient of the even part we have
\begin{align*}
    c_{n,e}&=\frac{1}{2\pi}\int_{-\pi}^\pi \hat{k}_e(\omega )\exp(-\omega in)d\omega\\
    &=-\frac{1}{2\pi}\int_{-\pi}^\pi \hat{k}_e(-u)\exp(uin)du\\
    &=-\frac{1}{2\pi}\int_{-\pi}^\pi \hat{k}_e(u)\exp(uin)du
\end{align*}
and the coefficient for the odd part is
\begin{align*}
    c_{n,o}&=\frac{1}{2\pi}\int_{-\pi}^\pi \hat{k}_o(\omega )\exp(-\omega in)d\omega\\
    &=-\frac{1}{2\pi}\int_{-\pi}^\pi \hat{k}_o(-u)\exp(uin)du\\
    &=\frac{1}{2\pi} \int_{-\pi}^\pi \hat{k}_o(u)\exp(uin)du
\end{align*}
Now Fourier series coefficients for analytic functions in a strip $[-a,a]$ decay as $O(\exp(-an))$, and since $\hat{k}_e$ and $\hat{k}_o$ are entire, $c_{n,e}$ annd $c_{n,o}$ decay as $O(\exp(-an))$ for all $a>0$. We then must have that $\frac{1}{2\pi}\int_{-\pi}^\pi \hat{k}_e(u)\exp(uin)du$ and $\frac{1}{2\pi}\int_{-\pi}^\pi \hat{k}_o(u)\exp(uin)du$ do, so that $k[n]$ does.
\fi

\gelumlpdtftsignaldecay
\begin{proof}
Follows immediately from Lemma \ref{lemma:gelu-mlp-layer-norm-entire} and Proposition \ref{prop:dtft-entire-signal-decay}.
\end{proof}

% \begin{proof}
% \begin{align*}
%     k[n]&=\frac{1}{2\pi}\int_{-\pi}^\pi \hat{k}(\omega)\exp(\omega in)d\omega 
% \end{align*}
% \textcolor{red}{What if we try letting $\tilde{k}$ be an extension of $\hat{k}$ to the entire real line that is still holomorphic but has the bound?}

% Note that
% \begin{align*}
%     \hat{k}(\omega)&=\sum_{l=1}^\infty c_l \phi_l(\omega)
% \end{align*}
% where $c_l$ decay exponentially. 

% \begin{align*}
%     \Tilde{k}(\omega)&=\begin{cases}
%     \hat{k}(\omega),\omega\in [-\pi,\pi]\\
%     \textrm{ some interpolation satisfying the second condition}
%     \end{cases}
% \end{align*}
% Then
% \begin{align*}
%     \vert k[n]\vert &=\vert \frac{1}{2\pi}\int_{-\pi}^\pi \hat{k}(\omega)\exp(\omega in)d\omega \vert\\
%     &\leq  \frac{1}{2\pi}\int_{-\infty}^\infty  \vert\Tilde{k}(\omega)\exp(\omega in)\vert d\omega \\
%     &\leq B\exp(-n b)
% \end{align*}
% \textcolor{red}{Fourier series}
% \textcolor{red}{treat DTFT as if it is a time signal, Fourier series of DTFT will somehow correspond to original time series.}
% \end{proof}

\subsection{SiLU: Proofs Related to Theorem \ref{thm:silu-mlp-signal-decay}}\label{sec:appendix-silu}

We first argue in Lemma \ref{lemma:silu-C-infinity} that the SiLU activation function is $C^\infty$. We then show in Proposition \ref{prop:silu-mlp-smooth} that SiLU MLPs with layer norm are $C^\infty$ and have integrable derivatives on compact domains. Next in Lemma \ref{lemma:dtft-integrable-function-bounded}, we argue that for an integrable DTFT, its inverse is bounded by a term proportional to the integral of the DTFT. In Proposition \ref{prop:dtft-smoothness-signal-decay}, we use the previous lemma to show that the DTFT being $N$ times differentiable implies a decay rate for the original signal. Finally, we prove our main result, that using a SiLU MLP to model a DTFT leads to faster than any polynomial rate in the time domain.

\begin{lemma}\label{lemma:silu-C-infinity}
    SiLU is $C^\infty$.
\end{lemma}
\begin{proof}
    The sigmoid function is $C^\infty$, as is the function $x$. The product of $C^\infty$ functions is $C^\infty$.
\end{proof}
\begin{proposition}\label{prop:silu-mlp-smooth}
A SiLU MLP mapping scalars to scalars with layer norm is $C^\infty$ with integrable derivatives on $[-\pi,\pi]$.
\end{proposition}
\begin{proof}
A SiLU MLP with layer norm involves finite linear combinations and finitely many compositions of $C^\infty$ functions, and is thus $C^\infty$. Now any SiLU MLP on a bounded domain has bounded derivatives of all orders (since they are continuous on a bounded domain). Thus, all derivatives are integrable on $[-\pi,\pi]$.
\end{proof}
\begin{lemma}\label{lemma:dtft-integrable-function-bounded}
    If the DTFT $\hat{k}\in L^1[-\pi,\pi]$, then $k$ is bounded and
    \begin{align*}
        \Vert k\Vert_\infty &\leq \frac{1}{2\pi }\Vert \hat{k}\Vert_1
    \end{align*}
\end{lemma}
\begin{proof}
This essentially follows the proof technique of Lemma 9.2.3 in \cite{heil2019introduction}, but in the reverse order and using the DTFT instead of the continuous Fourier transform. The idea is to express the signal as the inverse DTFT, which we can since $\hat{k}\in L^1[-\pi,\pi]$, and then use the fact that the values on the complex unit circle have magnitude $1$.
    \begin{align*}
        \vert k[n]\vert&=\left\vert \frac{1}{2\pi }\int_{-\pi}^\pi \hat{k}(\omega)\exp(i\omega n)d\omega \right\vert\\
        &\leq \frac{1}{2\pi }\int_{-\pi}^\pi \vert\hat{k}(\omega)\exp(i\omega n)\vert d\omega \\
        &=\frac{1}{2\pi }\int_{-\pi}^\pi \vert\hat{k}(\omega)\vert d\omega\\
        &=\frac{1}{2\pi }\Vert \hat{k}\Vert_1
    \end{align*}
\end{proof}

The next proposition describes how smoothness of the DTFT implies decay of a time domain signal. While there are many very related results in the literature (for instance, \cite{heil2019introduction} shows the opposite direction for the continuous Fourier transform using a very similar proof technique), we were not able to find exactly this result stated or proven rigorously. Thus we state and prove it.

\begin{proposition}\label{prop:dtft-smoothness-signal-decay}
    If the Nth derivative of DTFT $\hat{k}$ exists and is integrable on $[-\pi,\pi]$ then 
    \begin{align*}
        \vert k[n]\vert &\leq \frac{1}{ 2\pi\vert   n\vert ^N}\Vert \hat{k}^{(N)}\Vert_1
    \end{align*}
    for all $n\neq 0$.
    %the discrete time series $n^N f[n]\in l^1$. 
    
\end{proposition}
\begin{proof}

We first take the derivative of the DTFT
\begin{align*}
    \hat{k}(\omega)&=\sum_{m=-\infty}^\infty x[m]\exp(-i\omega m)\\
    \hat{k}'(\omega)&=\frac{1}{i}\sum_{m=-\infty}^\infty mx[m]\exp(-i\omega m).
\end{align*}

Since $\hat{k}$ is integrable over $[-\pi,\pi]$, we can plug it into the inverse DTFT
\begin{align*}
    \frac{1}{2\pi }\int_{-\pi}^\pi  \hat{k}'(\omega)\exp(i\omega n)&=\frac{1}{2\pi }\int_{-\pi}^\pi \frac{1}{i}\sum_{m=-\infty}^\infty mk[m]\exp(-i\omega m)\exp(i\omega n)d\omega \\
    &=\frac{1}{i}\sum_{m=-\infty}^\infty mk[m]\delta[n-m]\\
    &=\frac{n}{i}k[n]
\end{align*}
so that if $\hat{k}$ and $\hat{k}'$ are integrable, we obtain the key identity relating the inverse DTFTs of a DTFT and its derivative
\begin{align}
    \mathcal{F}^{-1} \{\hat{k}\}&=\frac{i}{n}\mathcal{F}^{-1} \{\hat{k}'\}\label{eqn:inv-fourier-recursion}.
\end{align}
Thus
\begin{align*}
    \vert k[n]\vert &\leq \frac{1}{\vert n\vert}\Big\vert \mathcal{F}^{-1}\{\hat{k}'\}[n]\Big\vert\\
    &\leq \frac{1}{n^2}\Big\vert \mathcal{F}^{-1} \{\hat{k}^{(2)}\}[n]\Big\vert & \textrm{ Eqn. \ref{eqn:inv-fourier-recursion}, since $\hat{k}^{(2)}$ integrable}\\
    &\leq \frac{1}{\vert n\vert^{N}}\Big\vert \mathcal{F}^{-1} \{\hat{k}^{(N)}\}[n]\Big\vert &\textrm{ applying recursively, since $N$th derivative integrable}\\
    &\leq \frac{1}{2\pi\vert n\vert^N}\big\| \hat{k}^{(N)}\big\|_1
\end{align*}
where the last line follows from Lemma \ref{lemma:dtft-integrable-function-bounded}.
\end{proof}

\silumlpdtftsignaldecay*
\begin{proof}
    This follows immediately from Proposition \ref{prop:silu-mlp-smooth} and Proposition \ref{prop:dtft-smoothness-signal-decay}.
\end{proof}

\subsection{Visualizations for Smoothness and Decay}\label{appendix:visualizations-smoothness-decay}

We visualize the frequency responses and the corresponding impulse responses generated by the frequency domain (FD) RPE under the three activation functions for which we have shown theory, with results predicted by theory. For a randomly initialized FD RPE with Gelu activations the impulse responses decay to approximately 0 by $n=30$: this is very rapid decay and the curves visually look like exponential decay. For a randomly initialized SiLU RPE, the resulting impulse responses are similar. For the ReLU case we show the generated filters from a trained FD TNN RPE from one of the TNN layers. We see the impulse responses visually decay to approximately 0 within the finite length of 512 points. This is a slower rate of decay than either of the previous two.%[ "Would DELETE this, I think could be explained by scale factor on output"... start at $h[0]$ ranging from approximately $-0.5$ to $0.5$, and by $n=10$ reach $h[10]\approx 0$. This is slightly slower but still fairly rapid decay, and it is visually clear that it is slower decay than in SiLU. ]

\begin{figure}
    \centering
    \includegraphics[width=0.65\linewidth]{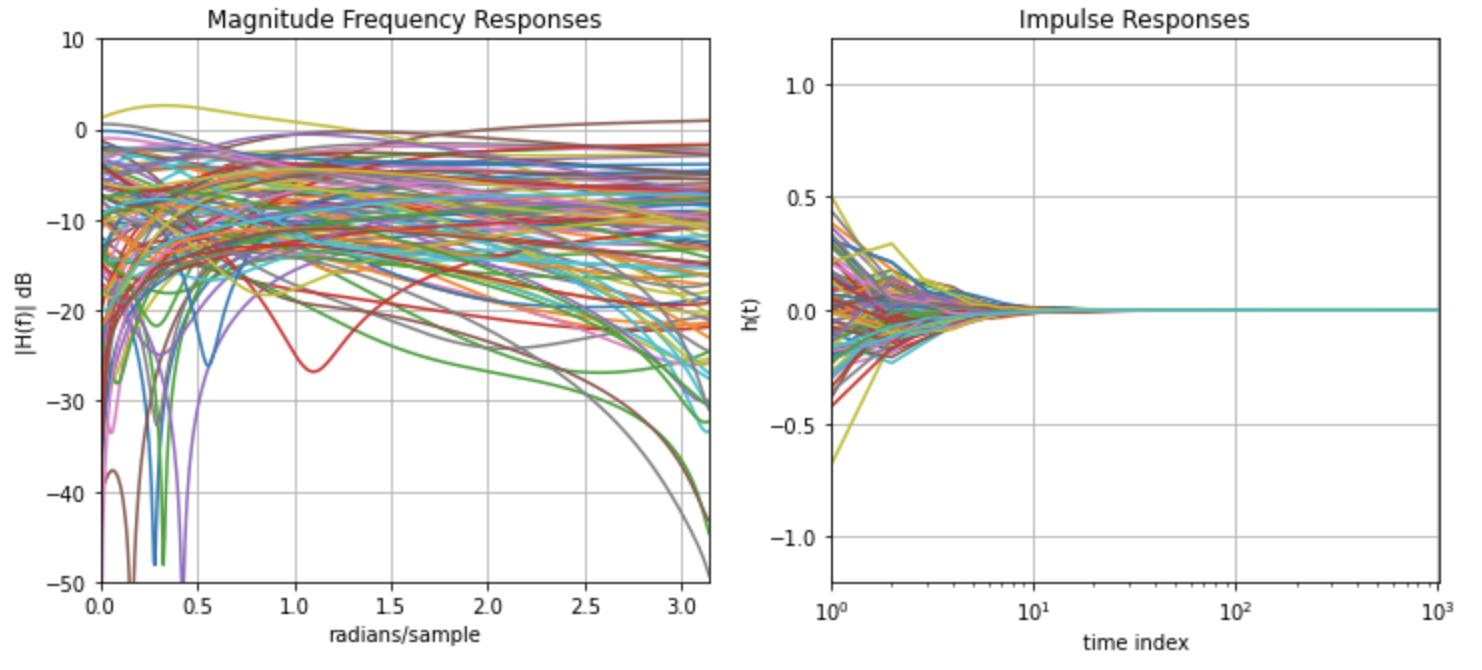}
    \caption{Frequency and impulse responses for a randomly initialized FD RPE MLP with \textbf{GeLU} activations. The curves on the left side are holomorphic, and theory predicts that the curves on the right hand will decay at faster than any exponential rate. They appear to decay approximately exponentially.}
    \label{fig:gelu-smoothness-decay}
\end{figure}

\begin{figure}
    \centering
    \includegraphics[width=0.65\linewidth]{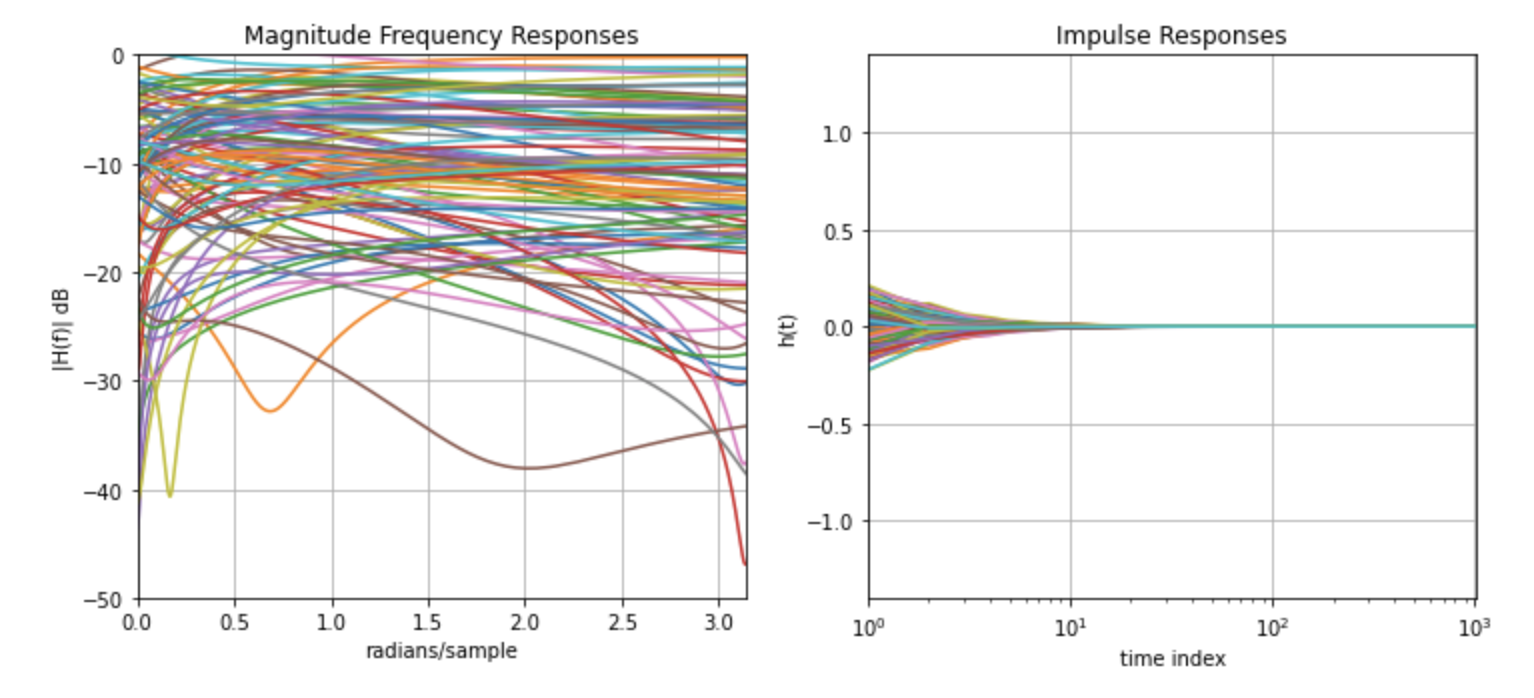}
    \caption{Frequency and impulse response for a randomly initialized FD RPE MLP with \textbf{SiLU} activations. The curves on the left side are $C^\infty$, and theory predicts that the curves on the right will decay at faster than any polynomial rate. They appear visually to have `almost' exponential decay.}
    \label{fig:silu-smoothness-decay}
\end{figure}

\begin{figure}
    \centering
    \includegraphics[width=0.65\linewidth]{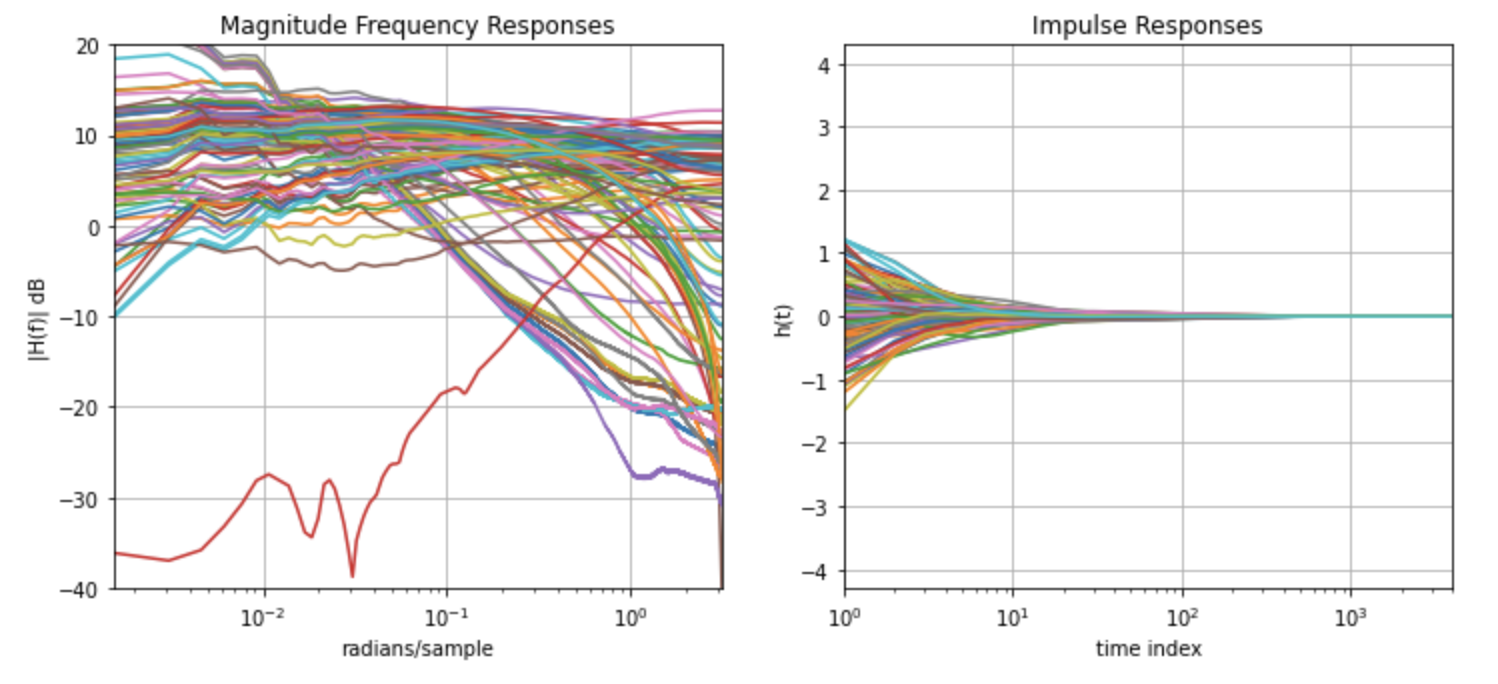}
    \caption{Frequency and impulse responses from an FD RPE MLP with \textbf{ReLU} activations, taken from one layer of a trained FD TNN. The curves on the left are continuous, and theory predicts that the curves on the right will be square summable. They clearly will vanish at infinity, although it is not immediately visually clear at what rate.}
    \label{fig:gelu-smoothness-decay}
\end{figure}

\section{Experiment Details and Additional Results}\label{sec:experimental-details-additional-results}
\subsection{Wikitext-103}

\subsubsection{Fourier Domain (FD-TNN)}\label{sec:fdd-tnn-wikitext}

For both causal and bidirectional models we use the default model and training hyperparameters from the TNN repository \ifx \footref{tnn_repo}\fi as the TNN baseline, defined in the first two columns in \cite{qin2023toeplitz} Table 13: LM (causal) and Roberta (bidirectional). One small HP discrepancy between the repository and table is the use of 7 decoder layers for the causal LM, which we used for all LM experiments, instead of the 6 they had in their paper.  We find that we can reduce the default number of RPE layers from 6 to 3 and improve the speed of the baseline with slight quality improvements. We provide these reproduced perplexity scores for the baseline in parenthesis in Table \ref{tab:wikitext-103-causal}, next to those reported by \cite{qin2023toeplitz}. 
For causal pretraining at a 512 sequence length, FD TNN achieves equivalent perplexity vs inference length as the TNN baseline (see Figure \ref{fig:alm_valid_ppl}a). We achieve between a 5 and 15 \% speed up for the causal case, and a nearly 80 \% speed up in the best case (6 RPE layers) for the bidirectional case. 

\begin{figure}
    \begin{subfigure}[t]{0.45\linewidth}
        \centering
        \includegraphics[width=0.9\linewidth]{figs/alm_extrap_all.png}
        \caption{Wikitext-103 Causal Pretraining.}
    \end{subfigure}
        \begin{subfigure}[t]{0.45\linewidth}
        \centering
        \includegraphics[width=0.9\linewidth]{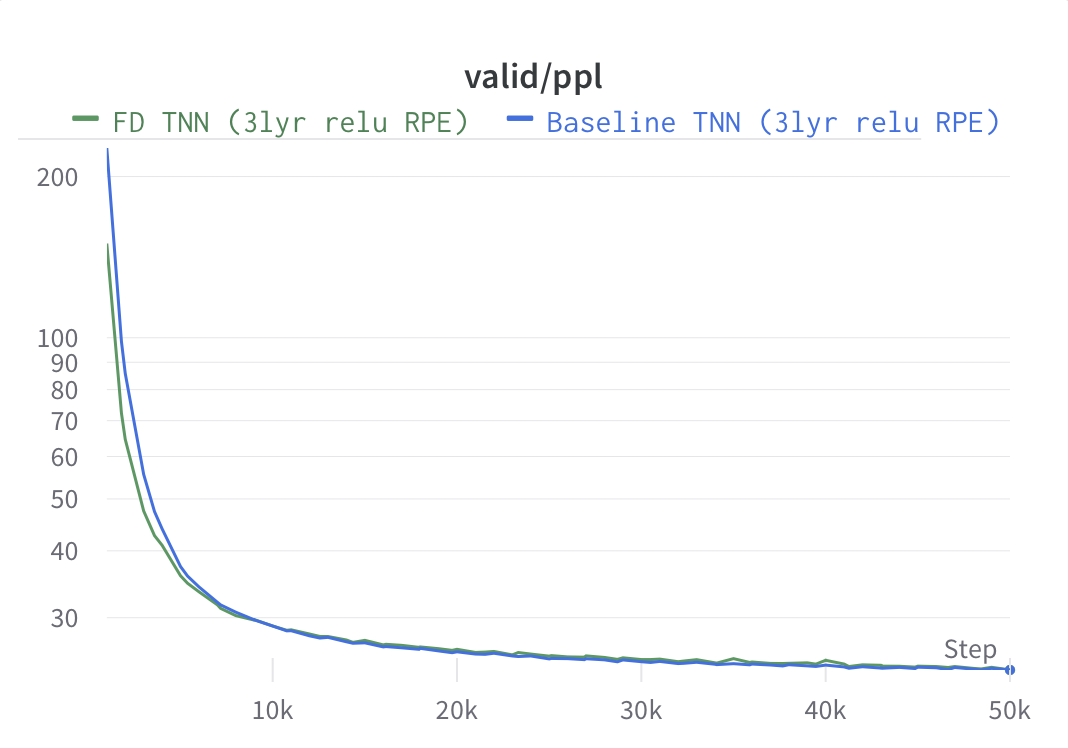}
        \caption{Wikitext-103 Causal Pretraining.}
    \end{subfigure}
    \centering
    \caption{a) In Wikitext-103 causal pretraining, our approach, FD TNN achieves equivalent perplexity vs inference length to TNN. b) Validation Perplexity vs iterations. In the causal setting, FD TNN converges to an equivalent quality at the same rate, but with a 5 to 15\% increase in training speed depending on the RPE MLP depth (see Figure \ref{fig:lra_bubble_and_extrapolation}). For these experiments, we used a learning rate 1e-3 for FD TNN and the default (5e-4) for the baseline.}
    \label{fig:alm_valid_ppl}%, with fast $\sim 10$-15\% training on a single A100 GPU at 512 sequence length. 
\end{figure}

\begin{figure}
    \begin{subfigure}[t]{0.45\linewidth}
        \centering
        \includegraphics[width=0.9\linewidth]{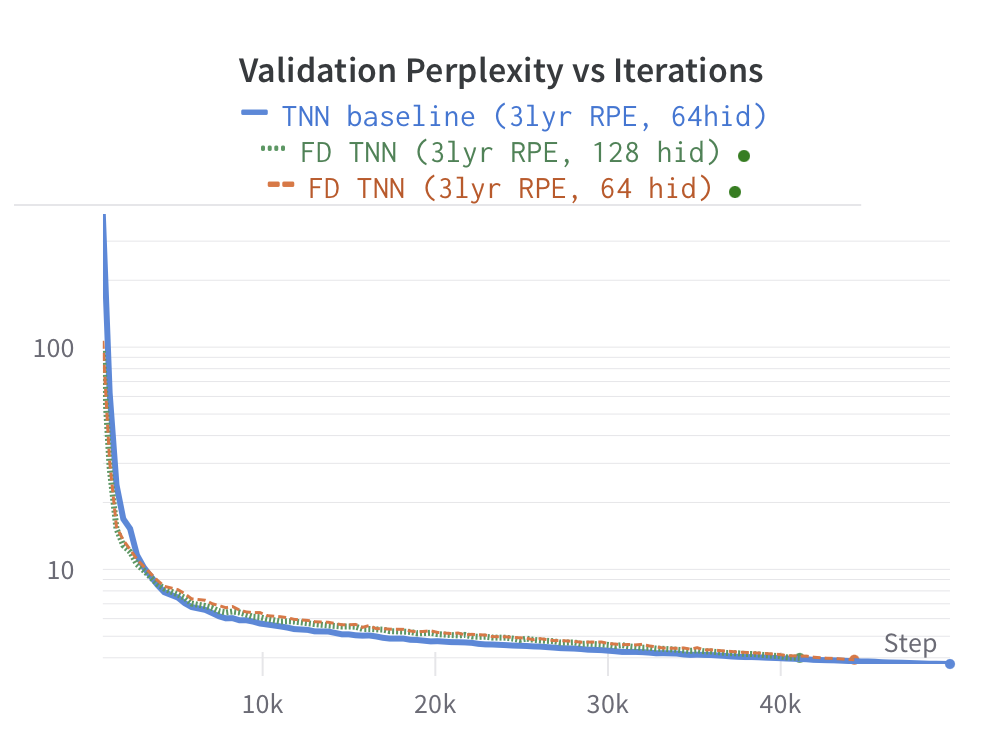}
        \caption{Wikitext-103 Bidirectional Pretraining.}
    \end{subfigure}
        \begin{subfigure}[t]{0.45\linewidth}
        \centering
        \includegraphics[width=0.9\linewidth]{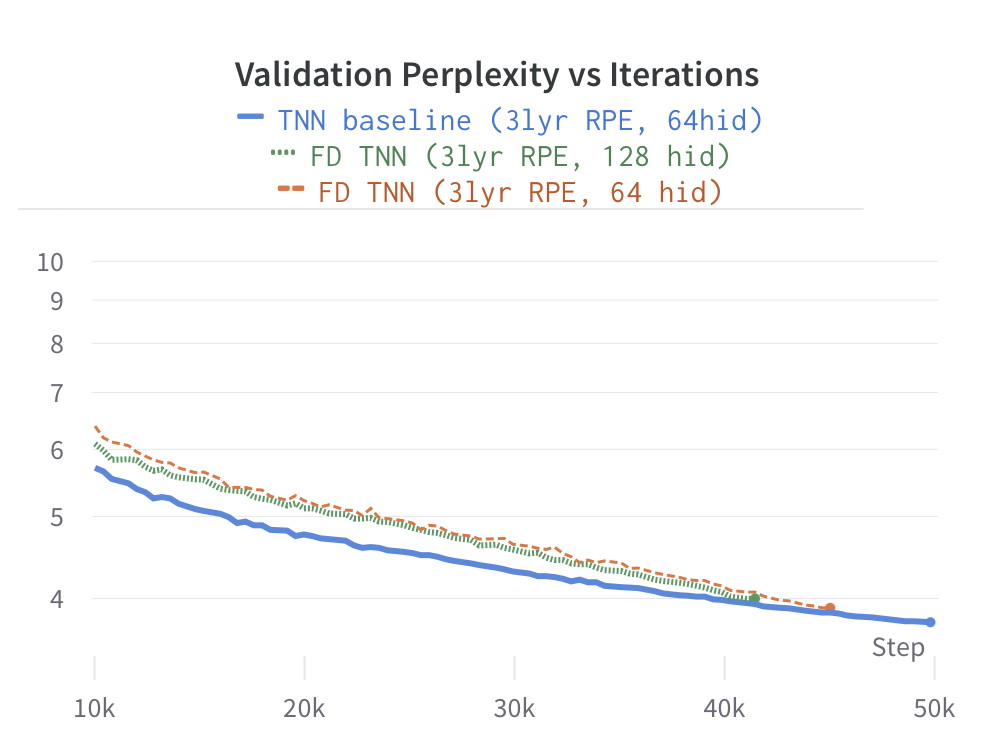}
        \caption{Wikitext-103 Bidirectional Pretraining (close up).}
    \end{subfigure}
    \centering
    \caption{a) In Wikitext-103 bidirectional pretraining, after minimal HP tuning from the default, we observed that FD TNN slightly lags the validation perplexity of the TNN baseline throughout much of the 50k training iterations, but closes this gap during the last 10k iterations. As a result, our 35-80\% speed up in iterations/sec (Figure \ref{fig:extrapolation}) applies to wall clock time assuming one trains for approximately 50k steps. For these results, we used a learning rate of 1e-3 for FD TNN and the default (5e-4) for the baseline. }
    \label{fig:blm_valid_ppl}%, with fast $\sim 10$-15\% training on a single A100 GPU at 512 sequence length. 
\end{figure}

\subsubsection{SKI-TNN}
We first train a bidirectional language model using the MLP free SKI-TNN for 50k steps and compare to the same baseline with a 3 layer RPE MLP (which shows improved model quality vs the 6 layer one) for validation perplexity. We see in Figure \ref{fig:blm_valid_ppl_ski}a) that the loss curves are very similar. In \ref{fig:blm_valid_ppl_ski}b), zooming in we see that both validation perplexities are close to $4$, although SKI TNN has slightly higher perplexity.

We then compare the SKI-TNN wall clock time and memory usage for sequence length 512 and 2048 to the 6 layer RPE TNN baseline, which was used in \cite{qin2023toeplitz}. Figure \ref{fig:blm_timing_memory_ski}a) shows approximately 25\% speedups compared to the baseline for sequence length 512 and 30\% speedups at sequence length 2048. Further, Figure \ref{fig:blm_timing_memory_ski}b) shows that our approach requires approximately 17\% and 42\% less memory for sequence lengths 512 and 2048, respectively.

Finally, we analyze the wall clock time and memory allocation for sparse plus low rank (our primary method) vs low rank only SKI. Figure \ref{fig:blm_timing_memory_ski_low_rank} shows the results. We find that the primary bottleneck in both cases is the low rank component, but that for wall clock time the sparse component (1d conv) still adds substantial overhead.

\begin{figure}
    \begin{subfigure}[t]{0.45\linewidth}
        \centering
        \includegraphics[width=0.9\linewidth]{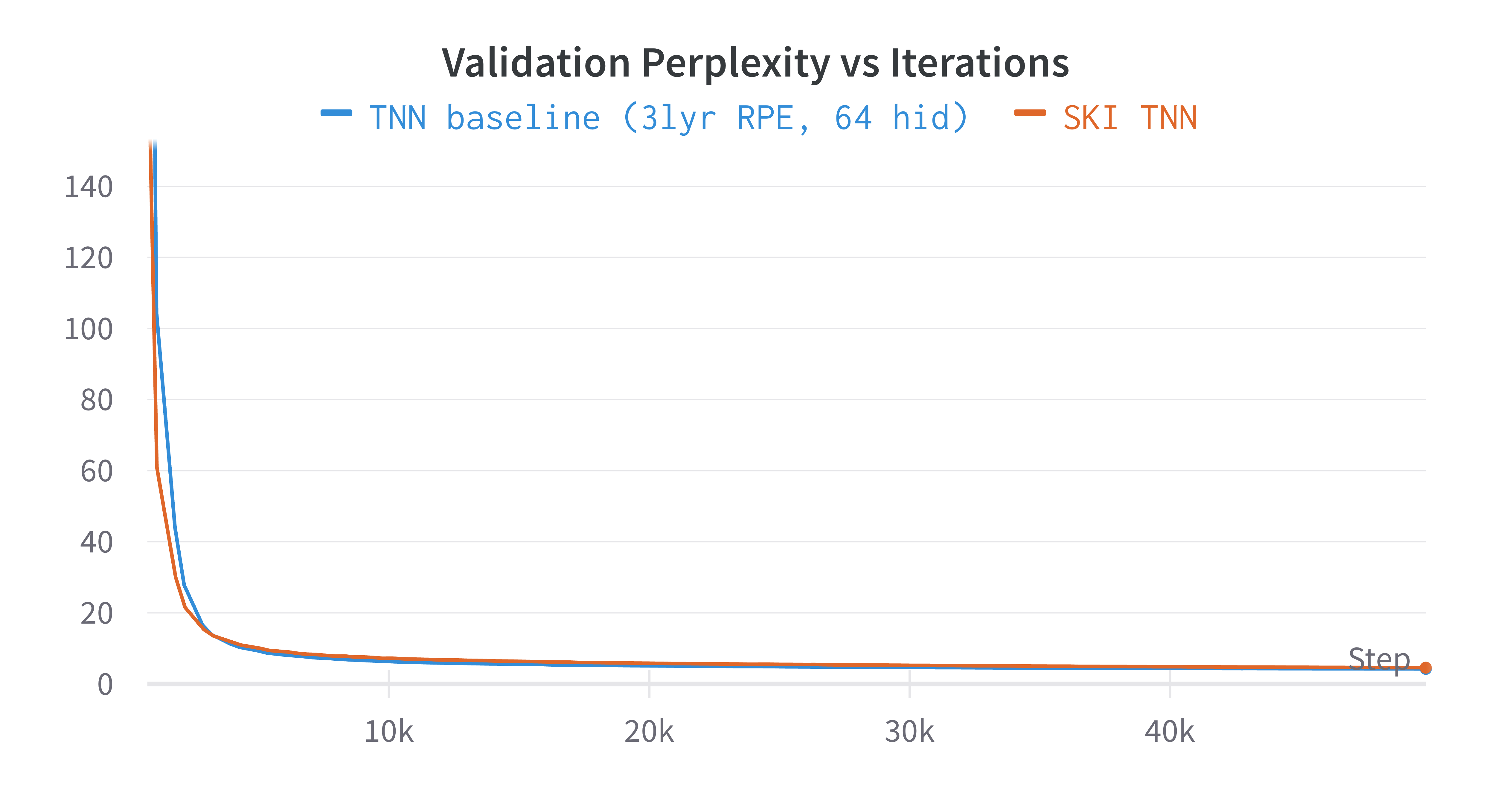}
        \caption{Wikitext-103 Bidirectional Pretraining, SKI.}
    \end{subfigure}
    \begin{subfigure}[t]{0.45\linewidth}
        \centering
        \includegraphics[width=0.9\linewidth]{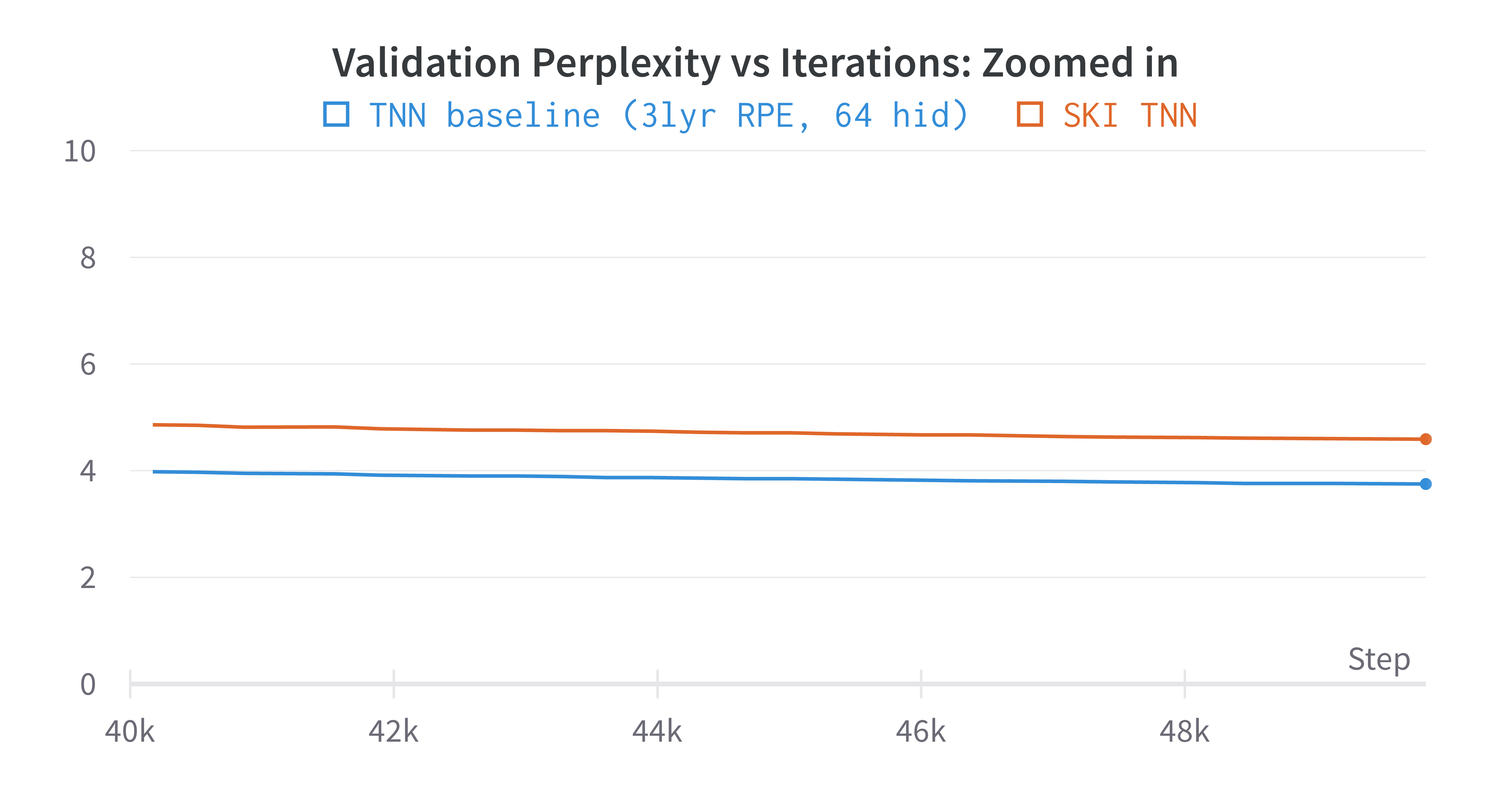}
        \caption{Bidirectional Pretraining, (close up).}
    \end{subfigure}
    \centering
    \caption{a) In Wikitext-103 bidirectional pretraining, using SKI leads to similar validation loss versus the baseline over 50k. b) We zoom in and see that the final validation perplexity is \textit{slightly} worse, which makes sense as SKI is approximate.}
    \label{fig:blm_valid_ppl_ski}%, with fast $\sim 10$-15\% training on a single A100 GPU at 512 sequence length. 
\end{figure}

\begin{figure}
    \begin{subfigure}[t]{0.45\linewidth}
        \centering
        \includegraphics[width=0.9\linewidth]{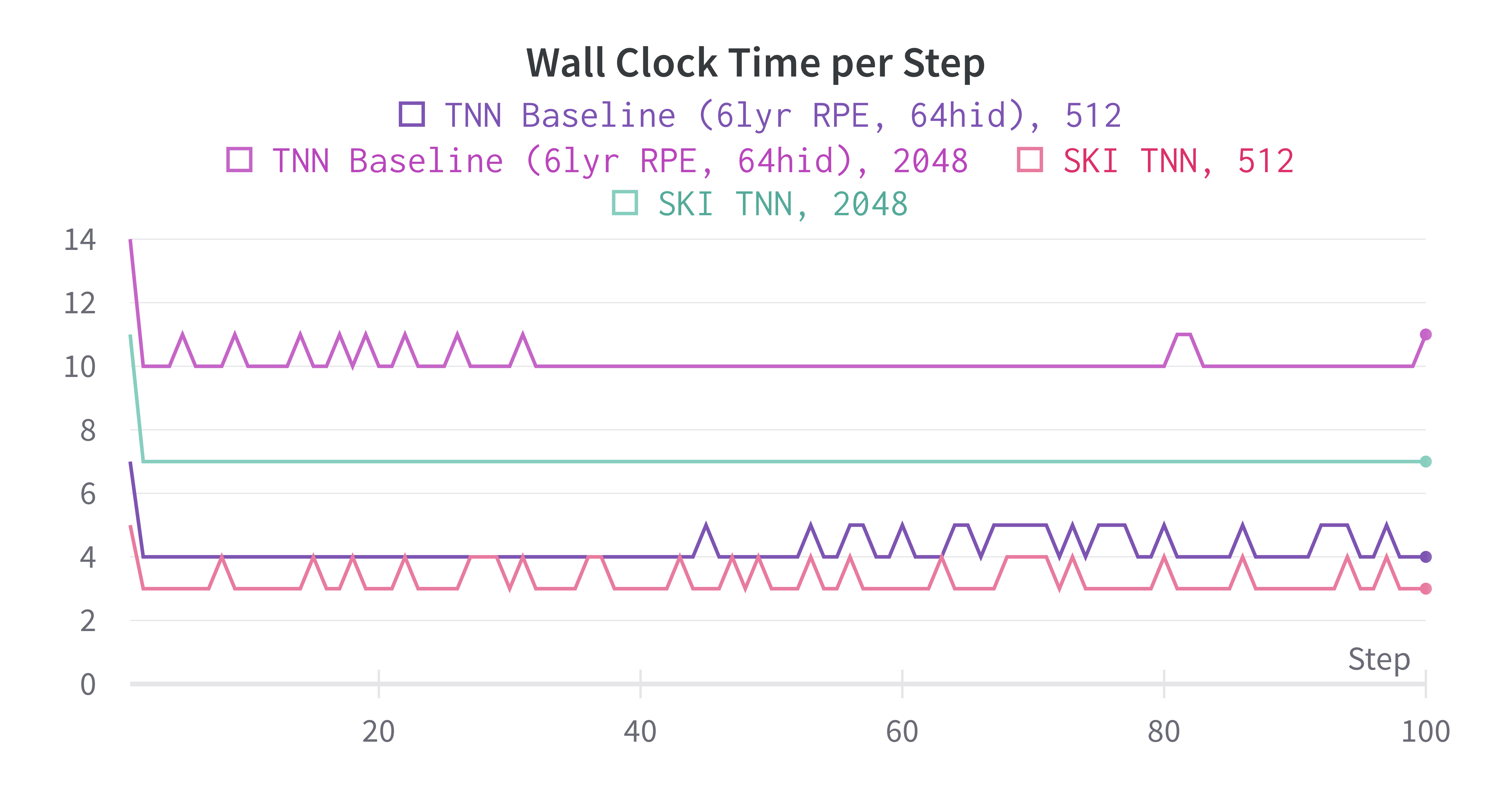}
        \caption{Wall clock time for SKI vs baseline.}
    \end{subfigure}
    \begin{subfigure}[t]{0.45\linewidth}
        \centering
        \includegraphics[width=0.9\linewidth]{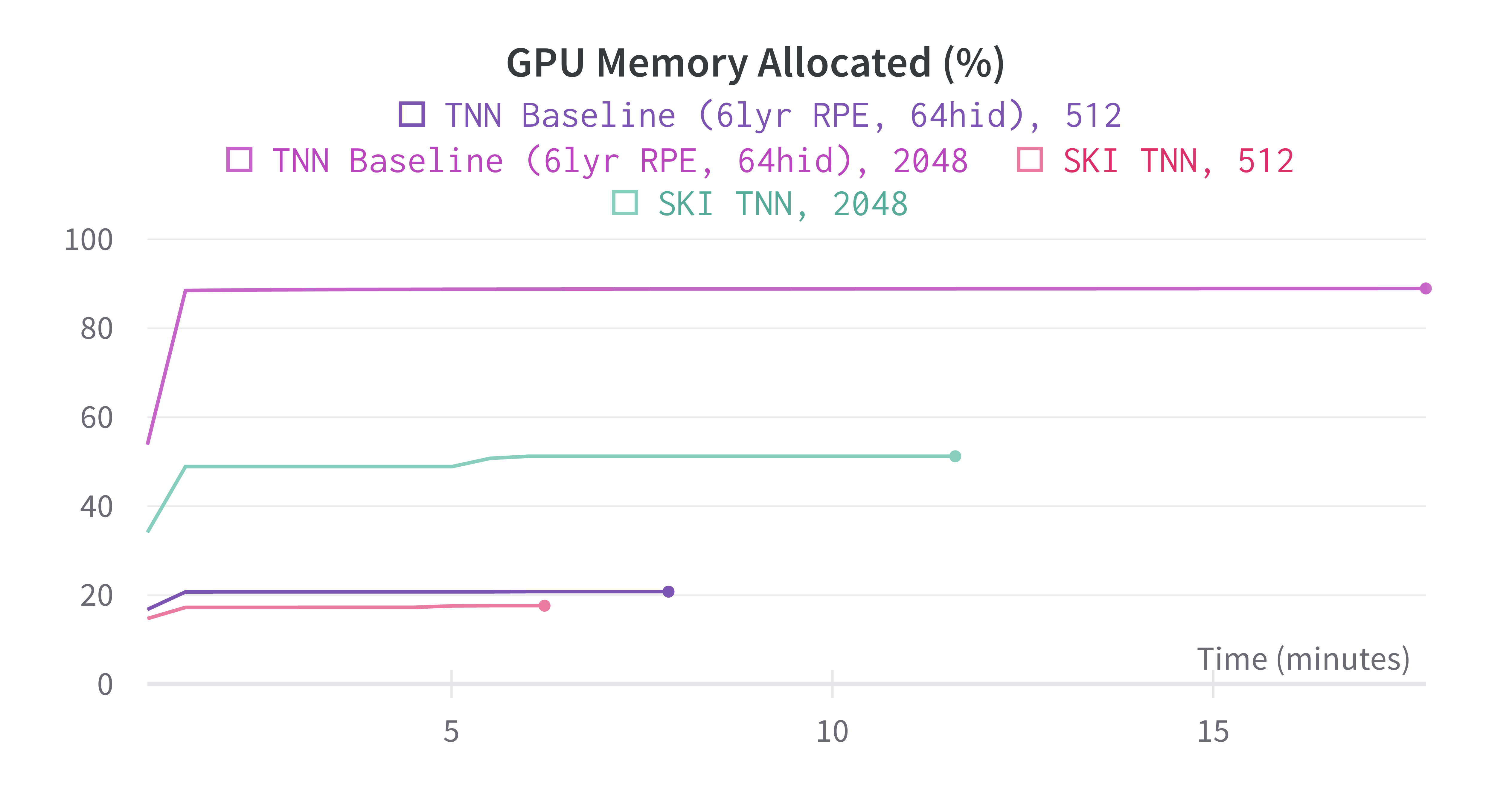}
        \caption{Memory usage for SKI vs baseline.}
    \end{subfigure}
    \centering
    \caption{a) Wall clock time per training step for SKI-TNN at sequence lengths 512 and 2048 vs the baseline with a 6 layer RPE from \cite{qin2023toeplitz}. At sequence length 512, our approach has a 25\% reduction in time per step in the flat regions. At 2048, our approach has a 30\% reduction in time per step. b) GPU memory allocation. At 512, our approach requires approximately 17\% less memory. At 2048, our approach requires approximately 42\% less memory.}
    \label{fig:blm_timing_memory_ski}%, with fast $\sim 10$-15\% training on a single A100 GPU at 512 sequence length. 
\end{figure}

% \begin{figure}
%     \begin{subfigure}[t]{0.45\linewidth}
%         \centering
%         \includegraphics[width=0.9\linewidth]{figs/wall-clock-ski-bidirectional.png}
%         \caption{Wall clock time for SKI vs baseline.}
%     \end{subfigure}
%     \begin{subfigure}[t]{0.45\linewidth}
%         \centering
%         \includegraphics[width=0.9\linewidth]{figs/gpu_memory_ski.png}
%         \caption{Memory usage for SKI vs baseline.}
%     \end{subfigure}
%     \centering
%     % \includegraphics[width=0.65\linewidth]{figs/alm_valid_ppl.png}
%     \caption{a) Wall clock time per training step for SKI-TNN at sequence lengths 512 and 2048 vs the baseline with a 6 layer RPE from \cite{qin2023toeplitz}. At sequence length 512, our approach has a 25\% reduction in time per step in the flat regions. At 2048, our approach has a 30\% reduction in time per step. b) GPU memory allocation. At 512, our approach requires approximately 17\% less memory. At 2048, our approach requires approximately 42\% less memory.}
%     \label{fig:blm_timing_memory_ski}%, with fast $\sim 10$-15\% training on a single A100 GPU at 512 sequence length. 
% \end{figure}

\begin{figure}
    \begin{subfigure}[t]{0.45\linewidth}
        \centering
        \includegraphics[width=0.9\linewidth]{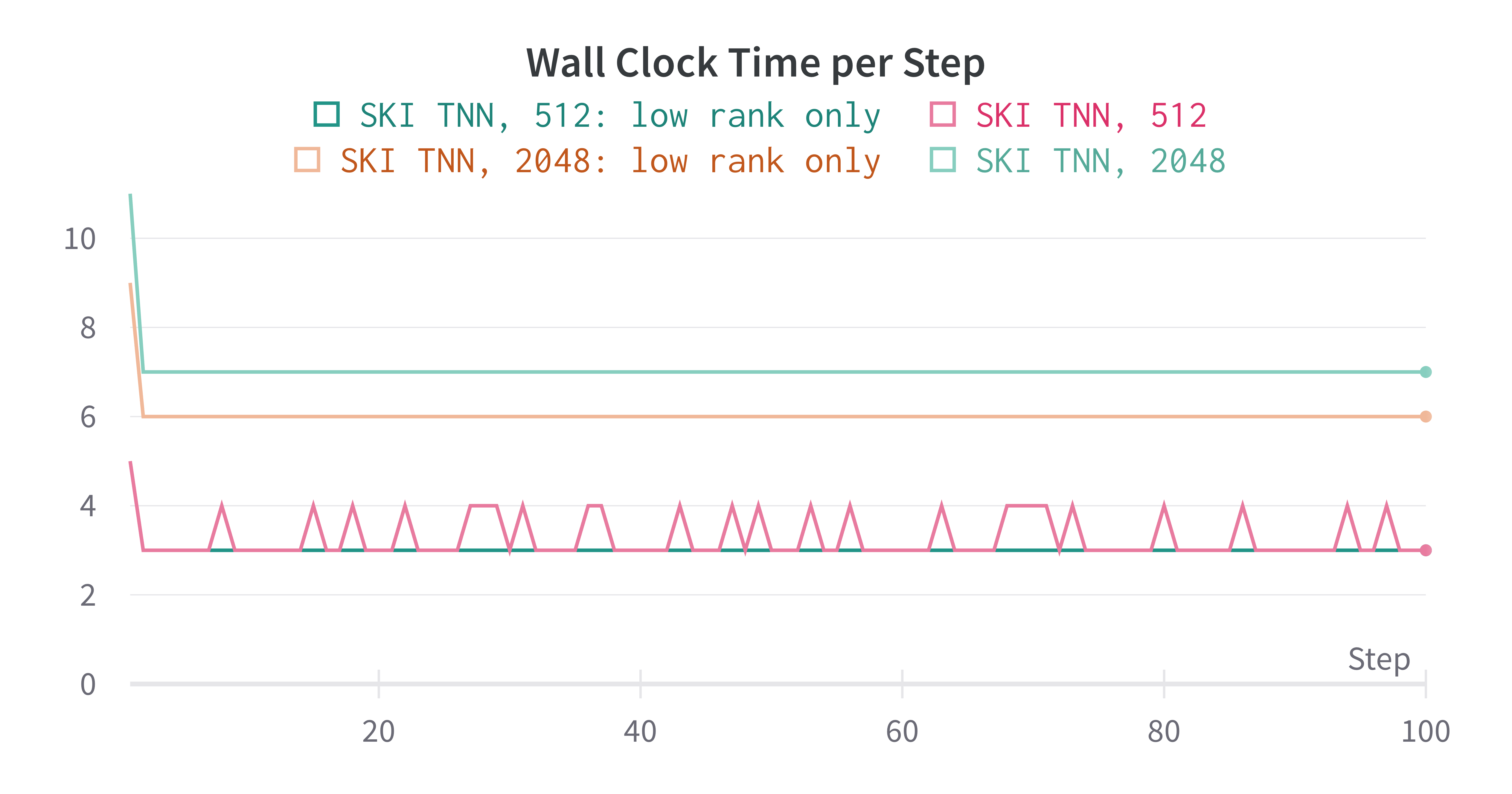}
        \caption{Wall clock time for SKI: low rank only vs sparse plus low rank.}
    \end{subfigure}
    \begin{subfigure}[t]{0.45\linewidth}
        \centering
        \includegraphics[width=0.9\linewidth]{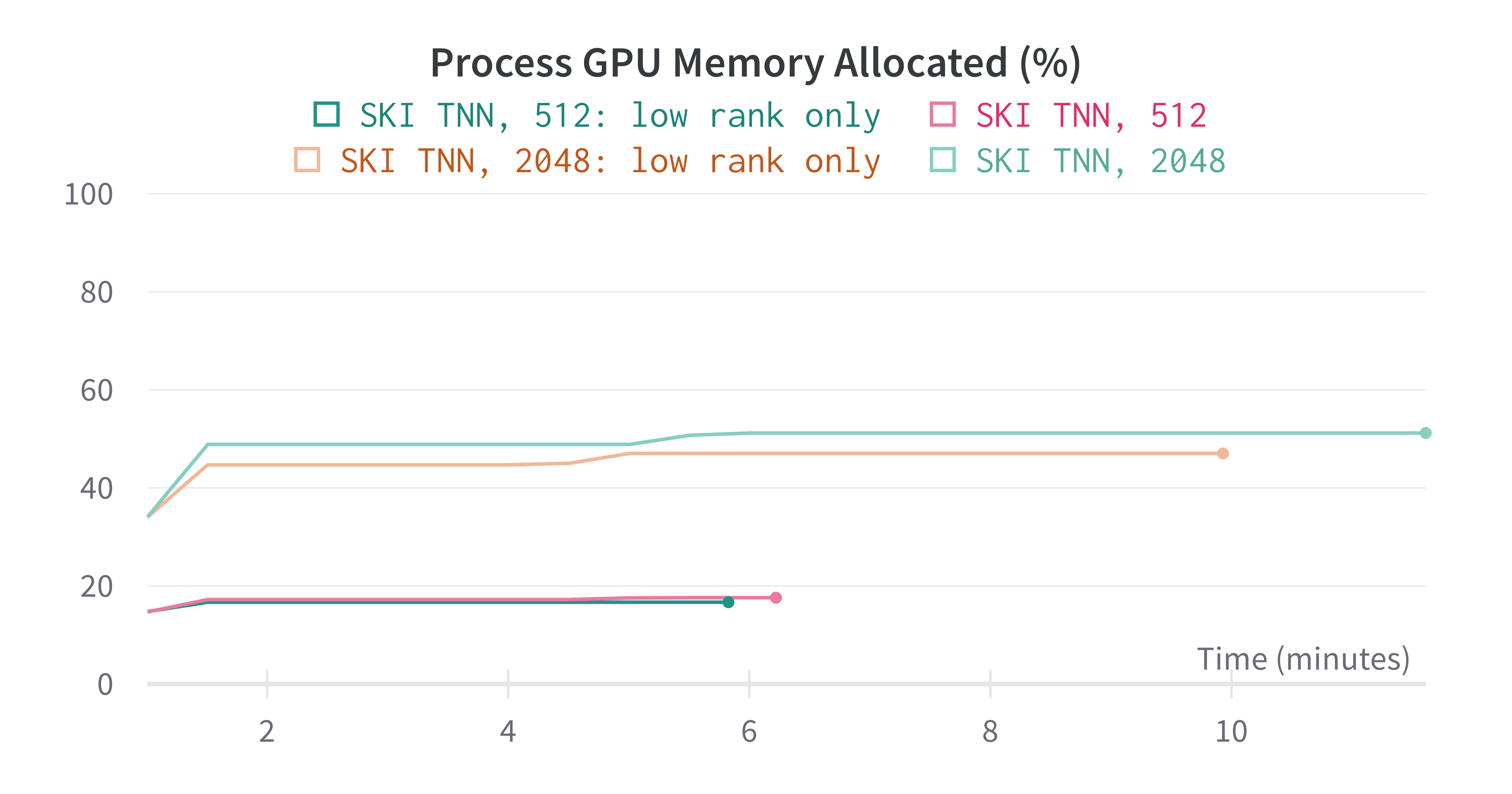}
        \caption{Memory usage for SKI:low rank only vs sparse plus low rank.}
    \end{subfigure}
    \centering
    \caption{a) Wall clock time per training step for SKI-TNN using either the low rank component only or both the sparse and low rank components. We see that the low rank component is the primary bottleneck, but the sparse component (the 1d conv) still adds a substantial amount of time. b) GPU memory allocation for low rank vs sparse plus low rank. Here, nearly all of the memory is due to the low rank component.}
    \label{fig:blm_timing_memory_ski_low_rank}%, with fast $\sim 10$-15\% training on a single A100 GPU at 512 sequence length. 
\end{figure}

% \begin{figure}
%     \centering
%     \includegraphics[width=0.65\linewidth]{figs/percent_speed_up.png}
%     \caption{Percent Speed Up in Iterations/Sec.  Note that we do not include our SKI based approach directly in this plot as it does not use an MLP based RPE and thus we cannot vary layer size. \textcolor{red}{same issue as above, just a bit too big, even for Appendix. Also, we should mention what the hardware here is, but that we display the relative speedup instead of absolute as in tends to be more stable across hardware.}}
%     \label{fig:percent_speed_up}
% \end{figure}

% \subsubsection{SKI with Inverse Time Warp}\label{sec:ski-tnn-wikitext}

% Add following plot
% \begin{itemize}
%     \item WPS or wall clock per iteration for
%     \begin{itemize}
%         \item 512, 1024, 2048, 4096
%         \item Ours vs TNN
%         \item We don't include the perplexity of the higher numbers as the cost is still prohibitive
%         \item Should we even include bidirectional perplexity? It seems like it's .
%     \end{itemize}
% \end{itemize}

\end{document}